\documentclass{article}

\usepackage{PRIMEarxiv}
\usepackage{algorithm}
\usepackage{setspace}
\usepackage{algcompatible}

\usepackage[utf8]{inputenc} 
\usepackage[T1]{fontenc}    

\usepackage{hyperref}       
\usepackage{url}            
\usepackage{booktabs}       
\usepackage{amsfonts}       
\usepackage{nicefrac}       
\usepackage{microtype}      
\usepackage{fancyhdr}       
\usepackage{graphics}       
\graphicspath{{media/}}

\usepackage{xcolor}         

\usepackage{subfigure}
\usepackage{booktabs} 
\usepackage{paralist}
\usepackage[numbers]{natbib}
\usepackage{hyperref}

\usepackage{amsmath}
\usepackage{amssymb}
\usepackage{mathtools}
\usepackage{amsthm}

\usepackage[capitalize,noabbrev]{cleveref}

\theoremstyle{plain}
\newtheorem{theorem}{Theorem}[section]
\newtheorem{proposition}[theorem]{Proposition}
\newtheorem{lemma}[theorem]{Lemma}

\theoremstyle{definition}
\newtheorem{definition}[theorem]{Definition}

\theoremstyle{remark}
\newtheorem{remark}[theorem]{Remark}

\newcommand{\R}{\mathbb{R}}
\newcommand{\A}{\mathcal{A}}
\newcommand{\X}{\mathcal{X}}
\newcommand{\ta}{\tilde{a}}

\newcommand{\E}{\mathbb{E}}
\newcommand{\bitem}{\begin{itemize}}
\newcommand{\eitem}{\end{itemize}}
\newcommand{\benum}{\begin{enumerate}}
\newcommand{\eenum}{\end{enumerate}}
\newcommand{\beq}{\begin{equation}}
\newcommand{\eeq}{\end{equation}}
\newcommand{\beqs}{\begin{equation*}}
\newcommand{\eeqs}{\end{equation*}}
\newcommand{\bas}{\begin{align*}}
\newcommand{\eas}{\end{align*}}

\usepackage[textsize=tiny]{todonotes}

\title{Bayes-Optimal Classifiers under Group Fairness
\thanks{This technical report has been largely superseded by our later paper: ``Bayes-Optimal Fair Classification 
with Linear Disparity Constraints via Pre-, In-, and Post-processing'' \citep{zeng2024bayesoptimal}. Please cite that one instead of this  technical report.}
}

\author{
  Xianli Zeng\\
University of Pennsylvania\\
  Philadelphia, PA\\
  \texttt{zengxl19911214@gmail.com} \\
   \And
  Edgar Dobriban \\
  University of Pennsylvania\\
  Philadelphia, PA\\
  \texttt{dobriban@wharton.upenn.edu} \\
     \And
  Guang Cheng \\
 University of California, Los Angeles\\
  Los Angeles, CA\\
  \texttt{guangcheng@ucla.edu} \\
}
\allowdisplaybreaks[4]

\begin{document}
\maketitle

\begin{abstract}
Machine learning algorithms are becoming integrated into more and more high-stakes decision-making processes, such as in social welfare issues. Due to the need of mitigating the potentially disparate impacts from algorithmic predictions, many approaches have been proposed in the emerging area of fair machine learning.
However, the fundamental problem of characterizing \emph{Bayes-optimal} classifiers under various group fairness constraints {has only been investigated in some special cases.} 
Based on the classical Neyman-Pearson argument \citep{NP1933,Shao2003} for optimal hypothesis testing, this paper provides a {unified} framework for deriving Bayes-optimal classifiers under group fairness. This enables us to propose a group-based thresholding method we call FairBayes\footnote{Codes for FairBayes are available at  \url{https://github.com/XianliZeng/FairBayes}.}, that can directly control disparity, and achieve an essentially optimal fairness-accuracy tradeoff. 
These advantages are supported by thorough experiments.
\end{abstract}

\section{Introduction}
\label{introduction}
With the rapid development of machine learning, algorithmic classifiers are increasingly applied in decision-making systems that have a long-lasting impact on individuals, including job applications, educational decision-making, credit lending and criminal justice. 
Accordingly, there are growing concerns about the disparate effects of machine learning algorithms.
Unfortunately, empirical studies have shown that machine learning algorithms focus mainly on utility, retaining or even amplifying implicit unfairness in historical data \citep{JJSL2016, BS2016, ZWY2017, SEC2019}. As a consequence, official institutions and organizations such as the White House \citep{WH2014,WH2016} and UNESCO \citep{UNESCO2020} advocate considering fairness in AI practice.

In response, several statistical metrics to quantify algorithmic disparity have been developed. For example, group fairness \citep{CFM2009,DHPT2012,HPS2016} targets to ensure various types of statistical parity across
distinct protected groups, while individual fairness \citep{JKMR2016,PKG2019,RBMF2020} aims to provide nondiscriminatory predictions
for similar individuals. 
In this paper, we focus on well studied group fairness criteria,
including  demographic parity \citep{CFM2009,KAS2012,CHS2020}, 
equality of opportunity \citep{HPS2016,ZLM2018,CHS2020}, predictive equality \citep{CSFG2017} and overall accuracy equality  \citep{BHJR2021}.

Recently, a large body of work has designed learning algorithms satisfying fairness criteria \citep{FFMS2015,KJ2016,CFDV2017,NH2018,ZLM2018,CFC2018,JL2019,CLKVN2019,XWYZW2019}.
In contrast, only a few papers discuss the fundamental problem of determining the Bayes-optimal classifiers  under fairness constraints. Equivalently, among all fair classifiers with respect to some criterion, which one is the most accurate?

\cite{CSFG2017} proved that, under several group fairness measures, the fair Bayes-optimal classifiers 
are group-wise thresholding rules with unspecified thresholds.  
\cite{MW2018} related demographic parity and equality of opportunity
to cost-sensitive
risks and derived fair Bayes-optimal classifiers under these two fairness measures.
In their work, the forms of the Bayes-optimal classifiers depend on some tuning parameters, and are only implicitly related to the specific fairness violation level. 
Under the limited setting of perfect demographic parity and equality of opportunity, exact forms of fair Bayes-optimal classifiers were derived in \cite{CCH02019} and \cite{SC2021}, respectively. 
Despite this recent theoretical progress,
 there is no systematic and unifying analysis of fair Bayes-optimal classifiers under group fairness measures, allowing both exact and approximate fairness.
 
In practice, it is important to \emph{directly} control the level of disparity or fairness violation.
However, the explicit dependence of the
optimal classifier on the unfairness level has not been elucidated. Moreover, Bayes-optimal classifiers with respect to group fairness measures such as overall accuracy equality have not been investigated in existing works. 
As a result, it is desirable  to establish a unified theoretical framework for Bayes-optimal fair classification problems, which allows directly controlling the level of disparity.

In this paper, we consider binary classification problems with protected attributes. 
We propose a unified framework for deriving Bayes-optimal classifiers that can handle  various fairness measures, by leveraging a novel connection with the Neyman-Pearson argument for optimal hypothesis testing.  
In particular, we derive the explicit and direct dependence of the fair Bayes-optimal classifier on the level of disparity for a binary protected attribute. 
This allows the user to directly specify "level of disparity" as an input parameter to the fair Bayes-optimal classifier, and control disparity at the specified level; which was not possible in prior work. We further derive the fair Bayes-optimal classifier for a multi-class protected attribute under perfect fairness in the Appendix.
Consistent with prior work, but much more generally, we find that thresholds for the optimal classifiers need to be adjusted \emph{for each protected group} to ensure statistical parity. These adjustments depend delicately on a number of problem characteristics, including the disparity level and the proportion of each group in the population. 

In summary, our theoretical approach goes beyong existing work in several aspects. First, we introduce a new proof strategy that connects fair classification with the well-established Neyman-Pearson argument. Second, we provide a unified approach that can handle many different fairness constraints. For some of these (especially overall accuracy equality), Bayes-optimal classifiers have not been studied before. 
Third, for a binary protected attribute, we establish the explicit dependence of the fair Bayes-optimal classifiers on an arbitrary disparity level  (in some cases more general than in prior works, which---as discussed above---require zero disparity). This allows users to \emph{directly control the level of unfairness/disparity}.
{Fourth, we do not impose any distributional assumptions on features so that the boundary case (where the features lay exactly on the “fair” boundary) does not necessarily have zero probability.  In this case, the optimal classifiers must be  carefully randomized (refer to the second term of \eqref{op-de-rule-dp} in Theorem \ref{thm-opt-dp}).
}

After deriving the fair Bayes-optimal classifiers, we design a group-based thresholding algorithm for fair classification; which we call FairBayes. 
Our first training step learns the feature-conditional probability of $Y=1$ for each protected group 
using traditional machine learning algorithms. 
In the second step, the optimal fair thresholds can be estimated by a one-dimensional search, using for instance the bisection method. 
We emphasize some notable
advantages of our FairBayes algorithm. 
First, 
the level of disparity is directly controlled. Second, compared to pre- and in-processing methods that require repeated training procedures for different fairness metrics or different levels of discrepancy, our post-processing method is computationally more efficient, as the constraints are handled only in the faster second step.

We summarize our contributions as follows.
\begin{itemize}
    \item We provide a unified framework for deriving Bayes-optimal classifiers under group fairness measures, including demographic parity, equality of opportunity, predictive equality and overall accuracy equality. The theoretical results can serve as a guideline for practical algorithm design.
    
    \item We propose a simple post-processing algorithm---called FairBayes---for binary fair classification within this theoretical framework.  The proposed FairBayes algorithm is theoretically justified and \emph{can directly control the disparity}. Moreover, as our classifier mimics Bayes-optimal classifiers, it can achieve a favorable fairness-accuracy tradeoff. 
    \item We demonstrate empirically that our FairBayes algorithm compares favorably to prior methods.
\end{itemize}

\section{Related Literature}\label{Literature}
Literature on fair machine learning algorithms has grown rapidly in recent years. In general, these algorithms can be broadly categorized  into three categories: pre-processing, in-processing and post-processing. Here we briefly introduce them and refer readers to \cite{SC2020} for a comprehensive review.

Pre-processing methods reduce biases implicit in the training data, and train learning algorithms on the debiased data. Examples include transformations \citep{FFMS2015,KJ2016,JL2019,CFDV2017}, fair representation learning \citep{ZWKT2013,CKYMR2016,CMJW2019} and fair generative models \citep{XYZ2018,SHC2019}. These methods are convenient to apply, as they do not change the training procedure. However, as argued in \cite{NARBSB2019}, biases could occur
even after pre-processing.

In-processing algorithms handle the fairness constraint during the training process. 
A commonly applied strategy is to incorporate fairness measures as a regularization term into the optimization objective \citep{GCAM2016,ZBR2017,NH2018,CLKVN2019,CJG2019,CHS2020}. However, one challenge is that fairness measures are often non-convex and even non-differentiable with respect to model parameters. This can make scalable training much more difficult or even impossible.
In the alternative approach of adversarial learning \citep{ZLM2018,CFC2018,XWYZW2019,LV2019}, 
the ability of the classifier to predict
the protected attribute is minimized. 
Although this can achieve promising results through careful design, its training process often lacks stability as a min-max optimization problem is considered \citep{CHS2020}. 
Other in-processing algorithms include 
domain-based training \citep{WQK2020}, 
where the protected attribute is explicitly encoded and its effect is mitigated.

Post-processing applies conventional learning algorithms to training data, 
and aims to mitigate disparities in the model output. 
The most common post-processing algorithm is the group-wise thresholding rule \citep{BJA2016,CSFG2017,VSG2018,MW2018,CCH02019,I2020,SC2021} that estimates the conditional probability of $Y=1$, 
and assigns individual thresholds to protected groups
aiming for statistical parity. 
Such methods have the advantage that they can, in principle, achieve fairness directly. In this paper, we propose a post-processing algorithm, FairBayes, to estimate the fair Bayes-optimal classifier under several group fairness measures. FairBayes is computationally efficient with  direct control on the level of disparity, and achieves  a favorable fairness-accuracy tradeoff.


\section{Preliminary and Notations}\label{Pre_and _Not}
In fair classification problems, two types of feature are observed: the usual feature $X\in\mathcal{X}$, and the protected (or, sensitive) feature $A\in\mathcal{A}=\{0,1\}$ \footnote{We consider a binary protected attribute in the main text and extend our results to multi-class protected attributes in the Appendix.}, with respect to which we aim to be fair. 
Here, we consider a binary classification problem with labels in $\mathcal{Y}=\{0,1\}$. 
For example, in a credit lending setting, 
$X$ may refer to common features such as 
 education level and income, $A$ may indicate the race or gender of the individual and $Y$ may correspond to the status of repayment or defaulting on a loan.
  
A randomized classifier outputs a prediction $\widehat{Y}\in\{0,1\}$ with a certain probability based on $X$ and $A$,
\begin{definition}[Randomized classifier]
  A randomized classifier is a measurable function
  $f:\mathcal{X}\times\{0,1\} \to [0,1]$, indicating the probability of predicting $\widehat{Y}=1$ when observing $X=x$ and $A=a$. We denote by $\hat{Y}_f=\hat{Y}_f(x,a)$ the prediction induced by the classifier $f$.
 \end{definition}
In this paper, we consider group fairness measures that require non-discrimination of protected groups. Here, we only introduce the definition of demographic parity. Other group fairness measures---equality of opportunity, predictive equality and overall accuracy equality---are introduced in Section \ref{sec:thm} of the Appendix.
  \begin{definition}[Demographic Parity]
  A classifier $f$ satisfies demographic parity if its prediction $\widehat{Y}_f$
   is statistically independent of the protected attribute $A$:
   $$P(\widehat{Y}_f  = 1|A = 1) =P(\widehat{Y}_f  = 1|A = 0).$$
  \end{definition}
  In applications, exact demographic parity may require a large sacrifice of accuracy, 
 and a ``limited'' disparate impact could be preferred. Similar to \citep{CHS2020}, we use the difference with respect to the demographic parity (DDP)
 to measure disparate impact:
$$\text{DDP}(f)=P(\widehat{Y}_f  = 1|A = 1) -P(\widehat{Y}_f  = 1|A = 0).$$

{\bf Notations.} We denote $p_a:=P(A=a)$; $p_{Y,a}:=P(Y=1|A=a)$; $\eta_{a}(x):=P(Y=1|A=a,X=x)$. 
Further, we denote by $P_X(x)$, $P_{X|A=a}(x)$ and $P_{X|A=a,Y=y}(x)$ the distribution function of $X$, the conditional distribution function of $X$ given $A=a$, and the conditional distribution of $X$ given $A=a,Y=y$, respectively. 

\section{Fair Bayes-Optimal Classifier}
Without the fairness constraint, the Bayes-optimal classifier, which minimizes the misclassification rate, is defined as
$f^\star\in\underset{f}{\text{argmin}} [P(Y\neq \widehat{Y}_f)]$; where recall that we predict $\hat Y_f=1$ with probability $f(x,a)$. A classical result  is the following \cite{DevroyeGL96, TsybakovBayes2004}:
\begin{proposition}\label{prop-ba-op} All Bayes-optimal classifiers $f^\star:\X\times \{0,1\}\to[0,1]$ have the form
$$f^\star(x,a)=I\left(\eta_a(x)>\frac12\right)+\tau_aI\left(\eta_a(x)=\frac12\right),$$
for all $(x,a)\in \X\times \{0,1\}$, where $I(\cdot)$ is the indicator function and $\tau_0,\tau_1\in[0,1]$ are two arbitrary constants. 
\end{proposition}

In each protected group, the Bayes-optimal classifier predicts the more likely class to maximize accuracy.
 
In this section, we derive the fair Bayes-optimal classifier with respect to demographic parity. We  consider the binary protected attribute case in the main text. 
We extend our theory to include fair cost-sensitive classification, multi-class protected attributes, as well as the other group fairness measures, including equality of opportunity, predictive equality and overall accuracy equality in Sections \ref{sec:thm_cost_sensitive} to \ref{sec:thm} of the Appendix. 
Denote by $\mathcal{F}_{D,\delta}$ the set of measurable functions satisfying the \emph{$\delta$- tolerance} constraint
$$\mathcal{F}_{D,\delta}=\{f: |\textup{DDP}(f)|\le \delta\}.$$
The $\delta$-fair Bayes-optimal classifier with respect to demographic parity is defined as
$$f_{D,\delta}^\star\in \underset{f\in\mathcal{F}_\delta}{\text{argmin}} [P(Y\neq \widehat{Y}_f)].$$
We define $D^\star=\sup_{f^*}\textup{DDP}(f^\star)$, where the supremum is taken over all Bayes-optimal classifiers $f^\star$  from Proposition \ref{prop-ba-op}.
\begin{theorem}\label{thm-opt-dp}
For any $\delta> 0$, all $\delta$-fair Bayes-optimal classifiers $f^\star_{D,\delta}$ have the following form:
\begin{compactitem}
    \item When $|D^\star|\leq \delta$,
    $f^\star_{D,\delta}$ can be any Bayes-optimal classifier $f^\star$ from Proposition \ref{prop-ba-op}.
    \item When $|D^\star|>\delta$, for all $x\in \mathcal{X}$ and $a\in \mathcal{A}$,
\begin{align}\label{op-de-rule-dp}
\begin{split}
 &f^\star_{D,\delta}(x,a)=
I\left(\eta_{a}(x)> \frac12+\frac{(2a-1)t^\star_{D,\delta}}{2p_a}\right)+ \tau_{D,\delta,a}^\star I\left(\eta_a(x)=\frac12+\frac{(2a-1)t^\star_{D,\delta}}{2p_a}\right),
\end{split}
\end{align}
where 
$t_{D,\delta}^\star$ is defined as
\begin{align}\label{t-star-dp}
\begin{split}
&\sup\left\{t:P_{X|A=1}\left(\eta_1(X)>\frac12+\frac{t}{2p_1}\right)>P_{X|A=0}\left(\eta_0(X)>\frac12-\frac{t}{2p_0}\right)+\frac{D^\star}{|D^\star|} \delta\right\};
\end{split}
\end{align}
and $(\tau_{D,\delta,1}^\star,\tau_{D,\delta,0}^\star)\in[0,1]^2$ are provided in the Appendix.
\end{compactitem}
\end{theorem}

\begin{remark}
Note that when $\eta_1(X)$ and $\eta_0(X)$ have  density functions on $[0,1]$, we have $P_{X|A=1}(\eta_1(X)=\frac12+\frac{t^\star_{D,\delta
}}{2p_1})=P_{X|A=0}(\eta_0(X)=\frac12-\frac{t^\star_{D,\delta
}}{2p_0})= 0$ and the optimal classifier is deterministic:
\begin{align}\label{op-de-rule-dp-degen}
\begin{split}
 &f^\star_{D,\delta}(x,a)=
I\left(\eta_{a}(x)> \frac12+\frac{(2a-1)t^\star_{D,\delta}}{2p_a}\right).
\end{split}
\end{align}
\end{remark}

\begin{figure}[t]
\begin{center}
\centerline{\includegraphics[width=1\columnwidth]{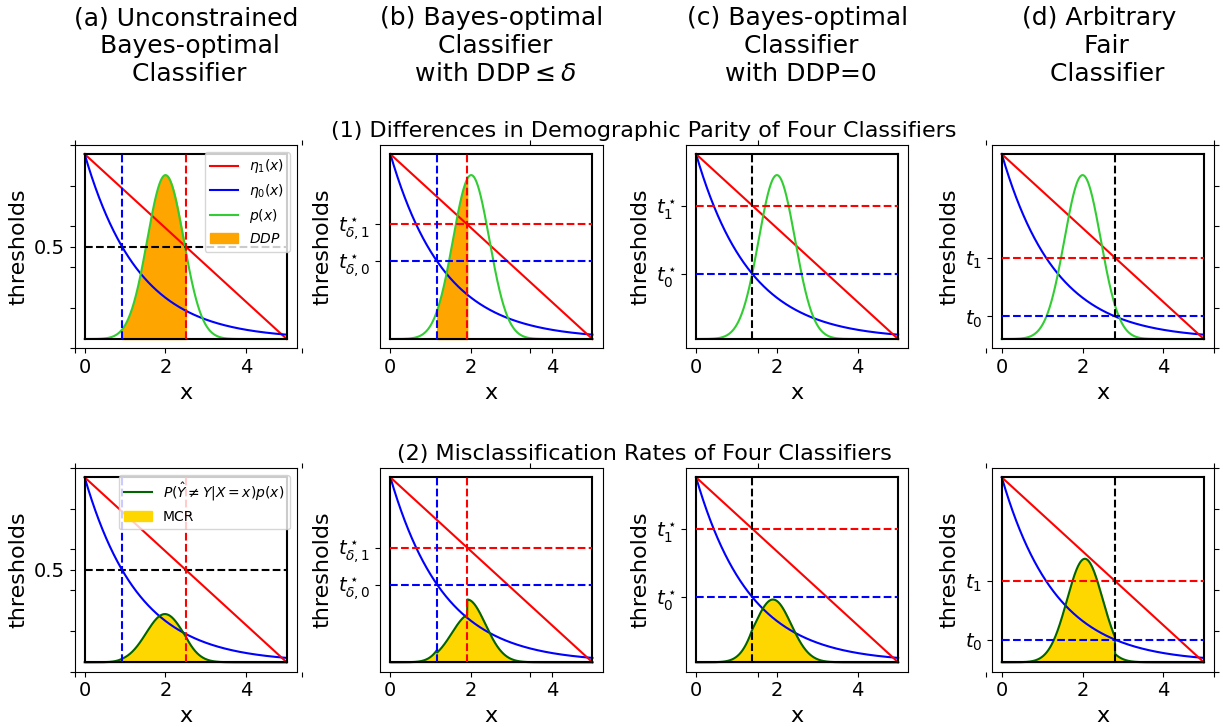}}
\caption{Difference in demographic parity (DDP) and misclassification rate (MCR) for four classifiers. 
Panel (a): the unconstrained Bayes-optimal classifier. 
Panel (b): a $\delta$-fair Bayes-optimal classifier.
Panel (c): fair Bayes-optimal classifier. 
Panel (d): non-optimal fair classifier. In the plots, the red and blue lines are the conditional probabilities of $Y=1$ in the two groups; the horizontal red and blue dashed lines are the thresholds for the groups (classifying as $1$ for $x$ less than a vertical dashed line); 
the light green curve is $x\mapsto p(x)$, the density function of $X$; 
the dark green curve is $x\mapsto p(x)P(\widehat{Y}\neq Y|X=x)$, whose integral is the misclassification rate.
The DDP and misclassification rate are
given by the shaded area in orange and yellow, respectively.}
\label{optimal-illustration}
\end{center}
\vskip -0.2in
\end{figure}

Figure \ref{optimal-illustration} illustrates our fair Bayes-optimal classifier. In the plot, the intersection points of any vertical line with two conditional probability densities of $Y=1$ for the two groups can induce the thresholds of a fair classifier, as shown in panel (d). 
The mis-classification rate is minimized for the thresholds from \eqref{t-star-dp}, as shown in panel (c).

Several comments for Theorem \ref{thm-opt-dp}  are in order. First, similarly to the unconstrained case, the fair Bayes-optimal classifier is a group-wise thresholding rule. Clearly, the accuracy is maximized by choosing the most likely set in each group. 
Second, mitigating disparity over protected groups necessitates shifting the thresholds. We need to ``balance" the thresholds for the two groups, by increasing the threshold for the group with a higher proportion classified as ``$1$'', thus bringing the proportions classified as ``$1$'' closer (and similarly for the ``$0$'' class). Moreover, the shift of each group depends on its size, since smaller shifts from the Bayes-optimal classifier maintain higher accuracy.

\section{FairBayes: Group-wise thresholding fair Bayes-optimal  classifier}
In this section, we propose the simple FairBayes post-processing algorithm (Algorithm \ref{alg:FBC1}), inspired by the  fair Bayes-optimal classifier. 
We consider the deterministic classifier \eqref{op-de-rule-dp-degen} and focus on learning the optimal thresholds $t^\star_{D,\delta}$. 
We observe data points $(x_i,a_i,y_i)_{i=1}^n$ drawn independently and identically from a distribution $\mathcal{D}$ over the domain
$\mathcal{X} \times \mathcal{A} \times \mathcal{Y}$. 

\textbf{Step 1.}
In the first training step, we apply standard machine learning algorithms to learn the feature-conditional probability of $Y=1$ for each protected group. Consider the function class $\mathcal{F}=\{f_\theta: \theta\in\Theta \}$ parametrized by $\theta$. 
Consider a loss function $L(\cdot, \cdot)$, for instance, $L_2$ or cross-entropy loss\footnote{Here we do not use the 0-1 loss, as minimizing the empirical 0-1 risk is generally not tractable. 
At the population level,
the minimizers of the risks induced by the 0-1, $L_2$ and cross-entropy losses are all equal to the
true feature-conditional probability density function of the label \citep{miller1993loss}.}.
The estimator $\widehat{\eta}$ of $\eta$, with $\eta(x,a)=\eta_a(x)$, is constructed by aiming to minimize the empirical risk\footnote{The empirical risk may be non-convex and hard to minimize. However, the main goal of this step is to construct an estimator $\widehat{\eta}$ of $\eta$; and any algorithm that achieves this is suitable. Here we choose empirical risk minimization due to its wide applicability.}
$
\widehat{\eta}\in \underset{f_\theta\in\mathcal{F}}{\text{argmin}}\frac1{n}\sum_{i=1}^{n}L(y_{i},f_{\theta}(x_{i},a_i))$.

\begin{algorithm}[tb]
   \caption{FairBayes under demographic parity}
   \label{alg:FBC1}
\begin{algorithmic}

   \STATE {\bfseries Input:} Tolerance level $\delta$. Dataset $S=S_1\cup S_0$ with  $S=\{x_{i},a_i,y_{i}\}_{i=1}^{n}$, $S_1=\{x_{1,i},y_{1,i}\}_{i=1}^{n_1}$ and $S_0=\{x_{0,i},y_{0,i}\}_{i=1}^{n_0}$.

   \STATE {\textbf{Step 1}:} Learn $\eta$:  $ \quad    \widehat{\eta}\in \underset{f_\theta\in\mathcal{F}}{\text{argmin}}\frac1{n}\sum_{i=1}^{n}L(y_{i},f_{\theta}(x_{i},a_i)).$

       \STATE {\textbf{Step 2}:} Find the optimal thresholds:
     \STATE{\quad Let $\widehat{D}(t)=\frac1{n_1}\sum_{i=1}^{n_1}{f}^t(x_{1,i},1) -\frac1{n_0}\sum_{i=1}^{n_0}{f}^t(x_{0,i},0)$, for $f^t$ from \eqref{ft},}

  \IF{$|\widehat{D}({0})|\leq \delta$}\quad
 $\widehat{t}_{D,\delta}=0$;
         
  \ELSIF{$\widehat{D}({0})> \delta$}\quad
   $\widehat{t}_{D,\delta}=\sup\left\{t :\widehat{D}(t)>\delta\right\}$;
         
   \ELSE\quad
   $\widehat{t}_{D,\delta}=\sup\left\{t :\widehat{D}(t)>-\delta\right\}$;
         
\ENDIF

\STATE {\bfseries Output:}  $\widehat{f}_{D,\delta}$, with
$\widehat{f}_{D,\delta}(x,a)=
I\left(\widehat\eta_{a}(x)> \frac12+\frac{n\widehat{t}_{D,\delta}}{2(2a-1)n_a}\right).$
\end{algorithmic}
\end{algorithm}
\textbf{Step 2.}
In the second step, we estimate the threshold for each protected group by solving the one-dimensional fairness constraint. As an example, we consider the constraint on demographic parity, $|\text{DDP}|\le \delta$. We also present algorithms for other group fairness measures in the Appendix. We divide the data into two parts,
according to the value of $A$:
$S_1=\{x_{1,i},y_{1,i}\}_{i=1}^{n_1}$, where $a_{1,i}=1$ and $S_0=\{x_{0,i},y_{0,i}\}_{i=1}^{n_0}$, where $a_{0,i}=0$. Based on Theorem~\ref{thm-opt-dp}, the following group-wise thresholding rule is considered:\footnote{We assume without further mention that we only divide by nonzero quantities. This can usually be ensured by restricting the range of parameters (e.g., $t$) considered. If not (e.g., if $n_0=0$ or $n_1=0$), our algorithm exits with an appropriate error message.}
\begin{align}\label{ft}
{f}^t(x,a)&=
I\left(\widehat\eta_{a}(x)> \frac12+\frac{n{t}}{2(2a-1)n_a}\right),
\end{align}
where $\eta_a$ and $p_a$ in \eqref{op-de-rule-dp} are replaced by their empirical (here, plug-in) estimators. 
Now, our goal is to construct an estimate $\widehat{t}_\delta$ such that the proposed classifier approximately satisfies the fairness constraint. 
{We 
define $\widehat{D}(t)$, an estimator of the unfairness measure $\text{DDP}(f^t)$, as
$$\widehat{D}(t)
=\frac1{n_1}\sum_{i=1}^{n_1}{f}^t(x_{1,i},1) 
-\frac1{n_0}\sum_{i=1}^{n_0}{f}^t(x_{0,i},0)
.$$}

When $|\widehat{D}(0)|\le\delta$, we set $\widehat{t}_{D,\delta}=0$, as---in the corresponding population case---the fairness constraint is satisfied. 
However, when $|\widehat{D}(0)|>\delta$, we follow the definition of $t^\star_{D,\delta}$ in \eqref{t-star-dp} to estimate it by
  $\widehat{t}_{D,\delta}=\sup\left\{t :\widehat{D}(t)>\widehat{D}({0})\delta/|\widehat{D}({0})|\right\}.$
Since $\widehat{D}(t)$ is monotone non-increasing as a function of $t$,  this can be effectively computed via either grid search or the bisection method.
Our final FairBayes estimator of the fair Bayes-optimal classifier $\widehat{f}_{D,\delta}$ outputs for all $x\in\X,a\in\A$,
$$\widehat{f}_{D,\delta}(x,a)=
\widehat{f}^{\,\hat{t}_{D,\delta}}(x,a)=
I\left(\widehat\eta_{a}(x)> \frac12+\frac{n\widehat{t}_{D,\delta}}{2(2a-1)n_a}\right).$$

FairBayes has several notable advantages. First, it is extremely easy to implement. In the first training step,
virtually all machine learning algorithms are suitable as no constraint is present. In the second step, the level of unfairness is directly specified and hyper-parameter tuning is not required. In addition, we will empirically show that this simple method can control disparity effectively. 
Second, as the FairBayes algorithm is inspired by the fair Bayes-optimal classifier, it is also highly accurate, as shown in our experiments. 
 
\vspace{-2mm}
\section{Experiments}
\vspace{-2mm}
\subsection{Synthetic data}
We first study a simple synthetic  dataset, in which the theoretical fair Bayes-optimal classifier can be derived explicitly to compare FairBayes with the theoretical benchmark.

\textbf{Data generating process.}
For a positive integer dimension $p$, let $X=(X_1,X_2,...,X_p)\in \mathbb{R}^p$ a generic feature, $A\in\{0,1\}$ be the protected attribute and $Y\in\{0,1\}$ be the label. 
Recalling the notations from the end of Section \ref{Pre_and _Not}, we generate $A$ and $Y$ according to the probabilities $p_1=0.7$, $p_{Y,1}=0.7$ and $p_{Y,0}=0.4$. Conditional on $A=a$ and $Y=y$, we generate $X$ from a multivariate
Gaussian distribution $N(\mu_{ay},\sigma^2I_p)$ (here $I_p$ is the $p$-dimensional identity covariance matrix). The entries of $\mu_{ay}$ are sampled from $\mu_{ay,j}\sim U(0,1)$, $j=1,\ldots,p$, where $U(0,1)$ is the uniform distribution over $[0,1]$, and $\sigma^2$ controls the variability  of the feature entries. 
Under the Gaussian setting,  both $\eta_1$ and $\eta_0$ have closed forms. As a result, we can calculate the $\delta$-fair Bayes-optimal classifier numerically using Theorems \ref{thm-opt-dp} and \ref{thm-opt-EO}. 

 \textbf{Experimental settings.} To evaluate our FairBayes algorithm, we randomly generate $20,000$ training data points  and $5,000$ test data points. 
  In step 1,  we employ logistic regression  to learn $\eta_1(\cdot)$ and $\eta_0(\cdot)$; as in the Gaussian case, the Bayes-optimal classifier is linear in $x$.
In step 2, we consider different  pre-determined unfairness levels $\delta$\footnote{We use $\delta_D$ for demographic parity and $\delta_E$ for equality of opportunity.} and  apply the bisection method to solve the equation for the thresholds. We denote by $\widehat{f}_{D,\delta}$ and $\widehat{f}_{E,\delta}$ the predictors obtained, aiming for Bayes optimality constrained with demographic parity and equality of opportunity, respectively. The empirical accuracy and unfairness measures are evaluated on the test set.


\begin{table}
\caption{Classification accuracies and levels of disparity for the fair Bayes-optimal classifier and FairBayes using logistic regression.}
\label{table_ddp}
\begin{center}
\setlength{\tabcolsep}{5.4pt}
\renewcommand{\arraystretch}{0.95}
\begin{small}
\begin{sc}
\begin{tabular}{cc|cc|cc|cc}
\hline
\multicolumn{4}{c|}{Demographic Parity}  & \multicolumn{4}{c}{Equality of Opportunity} \\\hline
\multicolumn{2}{c|}{Theoretical}  & \multicolumn{2}{c|}{FairBayes} &
\multicolumn{2}{c|}{Theoretical}& \multicolumn{2}{c}{FairBayes} \\ \hline
$\delta_D$& $\text{ACC}_{D,\delta}$ & DDP & $\text{ACC}_{D,\delta}$ &  $\delta_E$& $\text{ACC}_{E,\delta}$& DEO& $\text{ACC}_{E,\delta}$ \\ 
\hline
0.00 & 0.735 & 0.013(0.010) & 0.735(0.006) & 0.00 & 0.781  & 0.013 (0.010) & 0.779 (0.005) \\
0.10 & 0.755 & 0.100(0.016) & 0.755(0.006) & 0.04 & 0.781 & 0.043 (0.017) & 0.783 (0.006) \\ 
0.20 & 0.772 & 0.199(0.016) & 0.770(0.006) & 0.08 & 0.787 & 0.084 (0.018)& 0.785 (0.006)\\
0.30 & 0.782 & 0.299(0.015) & 0.781(0.006) &0.12 & 0.787 & 0.124(0.018) &  0.786(0.006)\\\hline\end{tabular}
\end{sc}
\end{small}
\end{center}
\end{table}


\begin{figure}
\begin{center}
\centerline{\includegraphics[width=0.95\columnwidth]{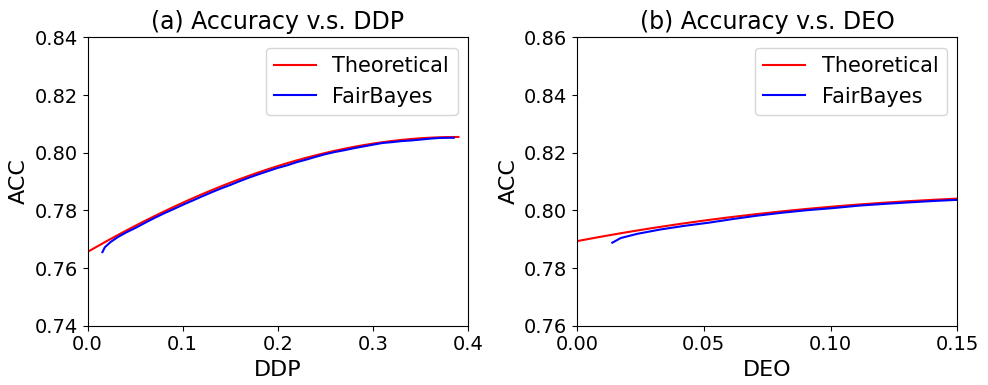}}
\caption{Fairness-Accuracy tradeoff of our classifier and the fair Bayes-optimal classifier with $p=10$, $\sigma=1$.  Panel (a): Tradeoff between accuracy and demographic parity. 
Panel (b): Tradeoff between accuracy and equality of opportunity parity.}
\label{plt-tradeoff-simu}
\vskip -0.2in
\end{center}
\end{figure}

Tables \ref{table_ddp} presents the test accuracies and disparity levels of the $\delta$-fair Bayes-optimal classifier and our estimator based on 100 random simulations\footnote{ Here, the randomness of the experiment is due to the random generation of the synthetic data.}. We present the result for $p=10,\sigma=1$; while the result for $p=2,\sigma=0.5$ is deferred to the Appendix. As shown in  Table \ref{table_ddp}, FairBayes controls 
 the pre-determined unfairness level $\delta$ as desired.
 We further present the fairness-accuracy tradeoff of the $\delta$-fair Bayes-optimal classifier and FairBayes in Figure \ref{plt-tradeoff-simu}. Our classifier closely tracks the behavior of the optimal one.

\subsection{Empirical Data Analysis}
{\bf Data Description.} Following \cite{CHS2020}, we consider three benchmark datasets in our experiments:  ``Adult'', ``COMPAS'' and ``Law School Admissions''. In \cite{CHS2020}, data were randomly split into training and test sets.
In our experiments, we further randomly split the training set into a training part (80\%) and a validation part (20\%). We present the results for the Adult dataset in the main text and the others in the Appendix. The target variable $Y$ is whether the income of an individual is more than \$50,000.  Age, marriage status, education level and other related variables are included in $X$, and the protected attribute $A$ refers to gender. 

{\bf Compared algorithms.}
\cite{CHS2020} introduced a kernel density estimation (KDE)-based constrained optimization method.
Their estimator of fairness measures is a differentiable function with respect to model parameters.
They further empirically demonstrated the superiority of their method over other baseline methods. Moreover, \cite{WQK2020}  designed a domain-independence based training technique, where a shared feature representation is learned for both subgroups. 
In their experiments, their method outperformed several baseline algorithms.
Our comparisons are mainly based on these two promising algorithms.  
\cite{I2020} proposed an  post-processing unconstrained optimization (PPUO) method for fair classification. We include their method as it is a  post-processing algorithm that is related  to ours.
Additionally, we also include other baseline methods such as domain discriminative training \citep{WQK2020}  and adversarial training \citep{ZLM2018}. 

{\bf Experimental settings.}  We follow the training settings in \cite{WQK2020}. For all datasets, a two-layer fully connected neural network with 32 hidden neurons is trained with the Adam optimizer with $(\beta_1,\beta_2)=(0.9,0.999)$, the default hyperparameters. 
For adversarial training \cite{ZLM2018}, we further use a linear classifier as the discriminator. In all cases, we train the model on the training set and select the one with best performance on the validation set. All  experiments use PyTorch; we refer readers to the Appendix for more training details, including learning rates, batch sizes and training epochs.  We repeat the experiments 20 times\footnote{The randomness of the experiments comes from the 
stochasticity of the batch selection in the optimization algorithm.}  for all the datasets.


We  first evaluate the FairBayes algorithm with various pre-determined levels of disparity. We present the simulation results in Table \ref{tab1}. We observe that FairBayes controls the disparity level at the pre-determined values, as desired.

\begin{table}[t]
 \caption{DDP of FairBayes with predetermined unfairness levels.}
 \label{tab1}
 \vskip -0.1in
 \begin{center}
 \begin{small}
 \begin{sc}
 \begin{tabular}{c|cccc}\hline
 $\delta_D $ & 0.00 & 0.04 & 0.08 & 0.12 \\\hline
DDP   &0.003 (0.003) & 0.040(0.004) & 0.078(0.004)  &0.116(0.002) \\\hline
\end{tabular}
 \end{sc}
\end{small}
 \end{center}
 \end{table}

We then compare our FairBayes algorithm with other baseline methods in Table \ref{table_adult}.  
As we can see, FairBayes is the most effective one for disparity control. 
It also has a satisfactory accuracy. 
Although slightly less accurate than other methods such as  post-processing unconstrained optimization and KDE based optimization, the discrepancy is almost negligible and perfect fairness comes at the cost of accuracy. Moreover, this discrepancy disappears when we enlarge the pre-determined unfairness level, as we show next.

\begin{table}[t]
 \caption{Classification accuracy and level of unfairness on the Adult dataset (Here, domain based algorithms and PPUO are designed for demographic parity).}
 \label{table_adult}
 \vskip -0.1in
 \begin{center}
 \begin{small}
 \begin{sc}
 \begin{tabular}{c|c|cc|cc}\hline
  \multicolumn{2}{c|}{}  &\multicolumn{2}{c|}{Demographic Parity}  & \multicolumn{2}{c}{Equality of Opportunity} \\\hline
 method &   parameter &ACC$_D$ & DDP & ACC$_E$& DEO\\\hline
 FairBayes & $\delta=0 $&0.832 (0.002) & 0.003 (0.003)& 0.849 (0.001)& 0.010 (0.007)\\
KDE based &$\lambda=0.75$&0.840 (0.004) &	0.054 (0.019)&	0.838 (0.010)&	0.032 (0.026)\\
Adversarial&$\alpha=3$& 0.784 (0.028) &	 0.072 (0.085) &0.756 (0.051)&	 0.068 (0.037)\\
PPUO && 0.848 (0.002)     & 0.110 (0.009)    &      &  
 \\
 Domain dis & & 0.828(0.009)	&0.032(0.025) &&	\\
 Domain ind& & 0.839 (0.008) &	0.077 (0.034)& &	\\
\hline
\end{tabular}
 \end{sc}
\end{small}
 \end{center}
 \vskip -0.1in
 \end{table}

To  further validate our FairBayes method, we compare its fairness-accuracy tradeoff with that of two baseline methods, KDE-based constrained optimization and adversarial-based training. 
For our FairBayes method, the level of unfairness is directly controlled. The range is from $0$ to the empirical DDP (or DEO) 
of the unconstrained classifier.
In constrained optimization, fairness and accuracy are balanced through a tuning parameter that controls the ratio between the loss and the fairness regularization term. 
We let this tuning parameter  $\lambda$ vary from $0.05$ to $0.95$ to explore a wide range of the tradeoff. 
In adversarial training, the tradeoff is controlled by changing the parameter $\alpha$ that handles the gradient of the discriminator. 
We vary this parameter from $0$ to $5$, as we empirically
find that in this range, the performance is representative and suffices for comparison (empirically, outside this range the accuracy  can drop quickly). More details about the effects of $\lambda$ and $\alpha$ can be found in \cite{WQK2020} and \cite{ZLM2018}, respectively.

Figure  \ref{plt-tradeoff-adult}  presents the fairness-accuracy tradeoff with respect to DDP (left panel) and DEO (right panel) evaluated on the Adult dataset. 
In the plot, each point represents a particular tuning parameter. 
Our FairBayes algorithm dominates other methods on this dataset. KDE-based optimization also achieves a satisfactory fairness-accuracy tradeoff.
However, it may lose some accuracy, 
due to its density estimation step, 
and its use of a  Huber surrogate loss to handle the non-differentiability of the absolute value function. 
With careful design, adversarial training can reduce the unfairness measure. However, its performance is unstable and the accuracy may drop quickly while improving the fairness measure. 

 In summary, our FairBayes algorithm achieves a  better accuracy-fairness tradeoff. 
 In addition, we emphasize its \emph{computational efficiency}. 
 For the other two methods, the tradeoff between fairness and accuracy is achieved through extensive training. 
 In contrast, our FairBayes method only necessitates a single training run, and the tradeoff can be controlled in a very efficient way. 
 For the Adult dataset,  FairBayes  takes only around 67 seconds to generate a DDP-accuracy tradeoff curve with 50 different disparity levels ({on a personal computer with an Intel(R) Core(TM) i9-9920X CPU @ 3.50Ghz and an NVIDIA GeForce RTX 2080 Ti GPU}). 
 However, it takes 1415 seconds for KDE-based optimization and 2279 seconds for adversarial training to derive a DDP-accuracy tradeoff curve with only 10 different tuning knobs ($\lambda$ for KDE-based optimization and $\alpha$ for adversarial training). Our method is two orders of magnitude faster, while achieving a better tradeoff.


\begin{figure}
\begin{center}
 \centerline{\includegraphics[width=0.95\columnwidth]{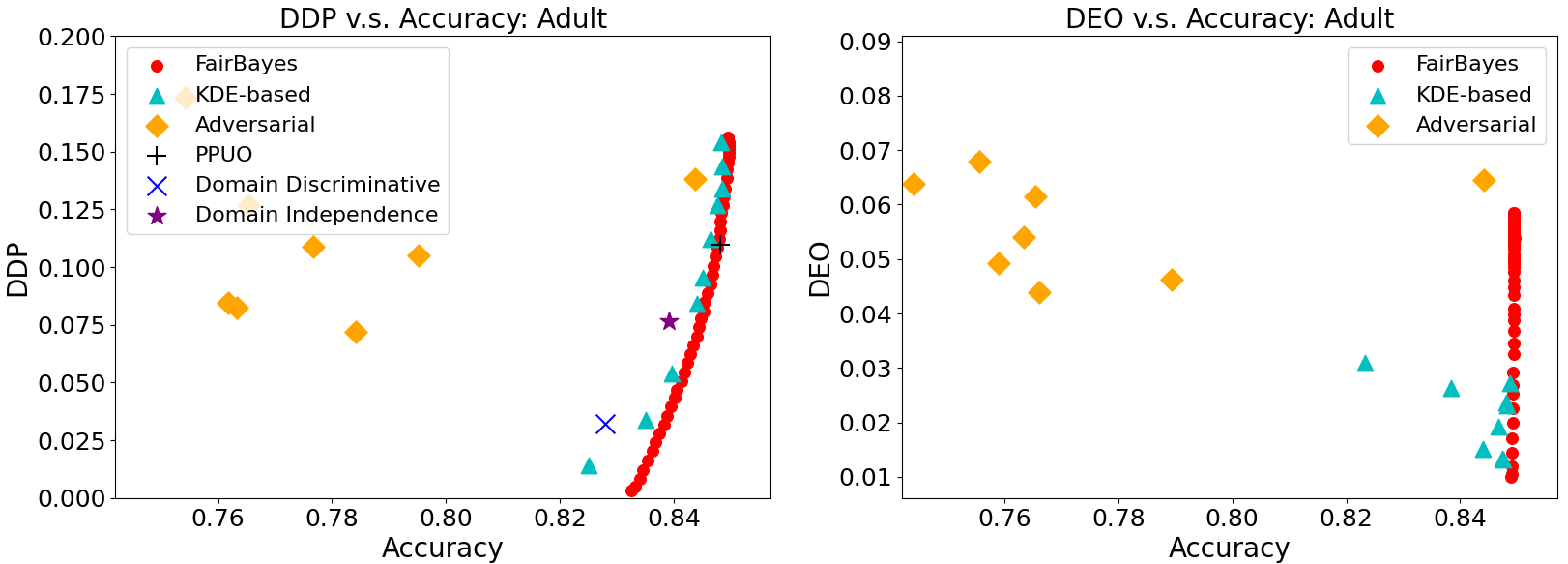}}
 \caption{Fairness-Accuracy tradeoff on ``Adult'' dataset. Left panel: Tradeoff with respect to demographic parity. Right panel: Tradeoff with respect to Equality of opportunity.}
 \label{plt-tradeoff-adult}
 \end{center}
  \vskip -0.3in
 \end{figure}

\vspace{-2mm}
\section{Summary and Discussion}
\vspace{-2mm}\label{sec:dis}
In this paper, we study Bayes-optimal classifiers under group fairness constraints.
Based on the classical Neyman-Pearson  argument in optimal hypothesis testing, we theoretically derive fair Bayes-optimal classifiers under group fairness constraints.
We then design the FairBayes post-processing algorithm for fair classification. 
FairBayes has several advantages. 
First, as FairBayes is based on a clear theoretical optimality principle, it can empirically achieve a superior fairness-accuracy tradeoff. 
Second, FairBayes is based on post-processing, thus the disparity can be controlled directly and precisely. Third, as the constraints are handled only in the second step where a one-dimensional constraint is solved, FairBayes is fast, especially when various possible fairness constraints are simultaneously considered (and the predictive model only needs to be trained once). 

This work provides many appealing directions for follow-up research.
 First, this paper only considers binary classification;
 the multi-class extension would be a significant problem. 
Second, based on our theoretical result, 
we demonstrate that a simple post-processing algorithm can achieve outstanding performance. As a result, it is of great interest to study how our theory can be combined with other algorithms. 
For instance, the optimal thresholds we derived may be used to find the optimal weights for re-weighting algorithms aiming to achieve certain fairness conditions \citep{KC2012}. 
In addition, the theory can serve as a guide for generating fair synthetic data. 

\section{Acknowledgements}

We are grateful to Robert C. Williamson and Nicolas Schreuder for providing the valuable references \cite{MW2018, SC2021}.



\bibliographystyle{plain} 
\bibliography{fair_optimal}


\newpage
\appendix

{\bf Additional notation(s).}
In this appendix, we use some additional notation.
For a real-valued function $f$ defined on $[a,b)$ for some $a<b$, we denote by $\lim_{x\to a^+} f(x)$ the limit from the right of $f$ at $a$, if it exists. Similarly, if $f$ is defined on $(b,a]$ for $b<a$,  we denote by $\lim_{x\to a^-} f(x)$ the limit from the left of $f$ at $a$, if it exists.

\section{Expressions for \texorpdfstring{$\tau^\star_{D,\delta,a}$}\. and Generalized Neyman-Pearson Lemma}
\subsection{Expressions for \texorpdfstring{$\tau^\star_{D,\delta,a}$}1}\label{sec-tau1}

Letting $t^\star_{D,\delta}$ be defined in \eqref{t-star-dp}, we define $\tau^\star_{D,\delta,a}$  as follows:

 Case (1). When $P_{X|A=1}\left(\eta_1(X)=\frac12+\frac{t^\star_{D,\delta}}{2p_1}\right)=P_{X|A=0}\left(\eta_0(X)=\frac12-\frac{t^\star_{D,\delta}}{2p_0}\right)=0$,\\ 
$(\tau_{D,\delta,1}^\star,\tau_{D,\delta,0}^\star)\in[0,1]^2$ can be arbitrary constants.

Case (2). When $P_{X|A=1}\left(\eta_1(X)=\frac12+\frac{t^\star_{D,\delta}}{2p_1}\right)>P_{X|A=0}\left(\eta_0(X)=\frac12-\frac{t^\star_{D,\delta}}{2p_0}\right)=0$,\\
$\tau_{D,\delta,0}^\star\in[0,1]$ can be  an arbitrary constant, and we set
\begin{equation*}
    \tau_{D,\delta,1}^\star=\frac{P_{X|A=0}\left({\eta_0}(X)\geq\frac12-\frac{t^\star_{D,\delta}}{2p_0}\right)+\frac{D^\star}{|D^\star|}-P_{X|A=1}\left(\eta_1(X)>\frac12+\frac{t^\star_{D,\delta}}{2p_1}\right)}{ P_{X|A=1}(\eta_1(X)=\frac12+\frac{t^\star_{D,\delta}}{2p_1})}.
\end{equation*}

Case (3). When $P_{X|A=0}\left(\eta_1(X)=\frac12-\frac{t^\star_{D,\delta}}{2p_0}\right)>P_{X|A=1}\left(\eta_1(X)=\frac12+\frac{t^\star_{D,\delta}}{2p_1}\right)=0$,\\
$\tau_{D,\delta,1}^\star\in[0,1]$ can be  an arbitrary constant,  and we set
\begin{equation*}
    \tau_{D,\delta,0}^\star=\frac{P_{X|A=1}\left({\eta_1}(X)>\frac12+\frac{t^\star_{D,\delta}}{2p_1}\right)-P_{X|A=0}\left(\eta_0(X)>\frac12-\frac{t^\star_{D,\delta}}{2p_0}\right)-\frac{D^\star}{|D^\star|}}{ P_{X|A=0}(\eta_0(X)=\frac12+\frac{t^\star_{D,\delta}}{2p_0})}.
\end{equation*}

 Case (4). When $P_{X|A=1}\left(\eta_1(X)=\frac12+\frac{t^\star_{D,\delta}}{2p_1}\right)>0$ and $P_{X|A=0}\left(\eta_0(X)=\frac12-\frac{t^\star_{D,\delta}}{2p_0}\right)>0$,\\
 we set $\tau_{D,\delta,1}^\star=0$ and
\begin{equation*}
    \tau_{D,\delta,0}^\star=\frac{P_{X|A=1}\left({\eta_1}(X)>\frac12+\frac{t^\star_{D,\delta}}{2p_1}\right)-P_{X|A=0}\left(\eta_0(X)>\frac12-\frac{t^\star_{D,\delta}}{2p_0}\right)-\frac{D^\star}{|D^\star|}}{ P_{X|A=0}(\eta_0(X)=\frac12+\frac{t^\star_{D,\delta}}{2p_0})}.
\end{equation*}

By construction, we have
 \begin{align*}
 \begin{split}
& P_{X|A=1}\left(\eta_1(X)>\frac{1}{2}+\frac{t^\star_{D,\delta}}{p_1}\right)+\tau^\star_{D,\delta,1}P_{X|A=1}\left(\eta_1(X)>\frac{1}{2}+\frac{t^\star_{D,\delta}}{p_1}\right)\\
=& \  P_{X|A=0}\left(\eta_0(X)>\frac{1}{2}-\frac{t^\star_{D,\delta}}{p_0}\right)+\tau^\star_{D,\delta,0}P_{X|A=0}\left(\eta_0(X)>\frac{1}{2}-\frac{t^\star_{D,\delta}}{p_0}\right).
\end{split}
 \end{align*}
 As a function of $t$, $S_{\eta_1}(\frac12+\frac{t}{p_1})$ is  right-continuous and $S_{\eta_0}(\frac12-\frac{t}{p_0})$ is left-continuous. Then, when  $P_{X|A=1}(\eta_1(X)=\frac12+\frac{t^\star_{D,\delta}}{2p_1})> 0$ or $P_{X|A=0}(\eta_0(X)=\frac12-\frac{t^\star_{D,\delta}}{2p_0})> 0$,
it can be verified that
 \begin{align*}
 \begin{split}
&P_{X|A=0}\left(\eta_0(X)>\frac{1}{2}-\frac{t^\star_{D,\delta}}{p_0}\right)+\frac{D^\star}{|D^\star|}\leq P_{X|A=1}\left(\eta_1(X)>\frac{1}{2}+\frac{t^\star_{D,\delta}}{p_1}\right)\\
\leq & \ P_{X|A=0}\left(\eta_0(X)\geq\frac{1}{2}-\frac{t^\star_{D,\delta}}{p_0}\right)+\frac{D^\star}{|D^\star|}\leq P_{X|A=1}\left(\eta_1(X)\geq\frac{1}{2}+\frac{t^\star_{D,\delta}}{p_1}\right).
\end{split}
 \end{align*}

Thus, we have that $0\leq\tau_{D,\delta,1}^\star,\tau_{D,\delta,0}^\star\le 1$ are well defined.

\subsection{Generalized Neyman-Pearson lemma}
Our argument relies on the following Generalized Neyman-Pearson lemma \citep{Shao2003}.
\begin{lemma}\label{NP_lemma}

Let $\phi_0,\phi_1, ..., \phi_{m}$ be $m+1$ real-valued functions defined on a Euclidean space $\mathcal{X}$. Assume they are $\nu$-integrable for a $\sigma$-finite measure $\nu$.
 Let $f^\star$ be any function of the form
\begin{align}
   \begin{split}
  f^\star(x)=\left\{\begin{array}{lcc}
   1,&& \phi_0(x)>\sum_{i=1}^mc_i\phi_i(x);\\
  \tau(x)&& \phi_0(x)=\sum_{i=1}^mc_i\phi_i(x);\\
   0,&& \phi_0(x)<\sum_{i=1}^mc_i\phi_i(x),
   \end{array}\right.
 \end{split}
 \end{align}
 where $0\le\gamma(x)\le 1$ for all $x\in \mathcal{X}$.
 For given constants $t_1, ...,t_m \in \mathbb{R}$, 
 let $\mathcal{F}$ be the class of Borel functions $f: \mathcal{X}\mapsto \mathbb{R}$  satisfying
 \begin{equation}\label{con1}
 \int_{\mathcal{X}} f\phi_i d\nu\leq t_i, \ \ \ i=1,2,...,m.
 \end{equation}
 and $\mathcal{F}_0$ be the set of functions in $\mathcal{F}$ satisfying \eqref{con1}  with all inequalities
 replaced by equalities. If $f^\star\in\mathcal{F}_0$, then
 $f^\star \in \underset{f\in\mathcal{F}_0}{\text{argmax}}\int_{\mathcal{X}}f\phi_0d\nu.$
 Moreover, if $c_i\ge 0$ for all $i=1,\ldots,m   $, then
 $f^\star \in  \underset{f\in\mathcal{F}}{\text{argmax}}\int_{\mathcal{X}} f\phi_0 d\nu.$
\end{lemma}


\section{Proof of Theorem \ref{thm-opt-dp}}\label{sec:proofA}
\begin{proof}
If $|D^\star|\le\delta$, we are done since $f^\star$ is the optimal rule over all measurable classifiers $f$.
Thus, we assume $|D^\star|>\delta$ from now on.
Let $f$ be a classifier that assigns $\widehat{Y}=1$ with probability $f(x,a)$ when observing $X=x$ and $A=a$. The mis-classification rate of $f$ is
\begin{align*}
\begin{split}
R(f)&=P(\widehat{Y}_f\neq Y)
=1-P(\widehat{Y}_f=1,Y=1)-P(\widehat{Y}_f=0,Y=0)\\
&=P(\widehat{Y}_f=1,Y=0)-P(\widehat{Y}_f=1,Y=1)+P(Y=1).
\end{split}
\end{align*}
Thus, minimizing the mis-classification rate with respect to $f$ is equivalent to maximizing  $P(\widehat{Y}_f=1,Y=1)-P(\widehat{Y}_f=1,Y=0)$, which can be expressed as
\begin{align*}
\begin{split}
&P(\widehat{Y}_f=1,Y=1)-P(\widehat{Y}_f=1,Y=0)
=\E\left[I(\widehat{Y}_f=1,Y=1)-I(\widehat{Y}_f=1,Y=0)\right]\\
=&\E\left\{\E\left[\left.I(\widehat{Y}_f=1,Y=1)-I(\widehat{Y}_f=1,Y=0)\right|X,A\right]\right\}\\
=&\int_{\mathcal{A}}\int_{\mathcal{X}}f(x,a)\eta_a(x)-f(x,a)(1-\eta(x,a))dP_{X,A}(x,a)\\
=&\int_{\mathcal{A}}\int_{\mathcal{X}}f(x,a)(2\eta_a(x)-1)dP_{X,A}(x,a).
\end{split}
\end{align*}


Next,
for any classifier $f$, we have,
\begin{align*}\label{constraint1}
\begin{split}
\textup{DDP}(f)&=P(\widehat{Y}_f=1|A=1)-P(\widehat{Y}_f=1|A=0)\\
&=\int_{\mathcal{X}}f(x,1)dP_{X|A=1}(x)-\int_{\mathcal{X}}f(x,0)dP_{X|A=0}(x)\\
&=\sum_{a\in\mathcal{A}}\left\{I(a=1)\int_{\mathcal{X}}f(x,a)dP_{X|A=a}(x)-I(a=0)\int_{\mathcal{X}}f(x,a)dP_{X|A=a}(x)\right\}\\
&=\int_{\mathcal{A}}\int_{\mathcal{X}}\frac{I(a=1)}{p_1}f(x,a)dP_{X|A=a}(x)-\int_{\mathcal{A}}\int_{\mathcal{X}}\frac{I(a=0)}{p_1}f(x,a)dP_{X|A=a}(x)dP_A(a)\\
&=\int_{\mathcal{A}}\int_{\mathcal{X}}f(x,a)\left[\frac{I(a=1)}{p_1}-\frac{I(a=0)}{p_0}\right]dP_{X,A}(x,a)\\
&=\int_{\mathcal{A}}\int_{\mathcal{X}}f(x,a)\left(\frac{2a-1}{p_a}\right)dP_{X,A}(x,a).
\end{split}
\end{align*}
The fourth equation holds since for any function $g$ defined on $\mathcal{A}$, we have $\int_{\mathcal{A}}g(a)dP_A(a)=p_1g(1)+p_0g(0)$.

This shows that the $\delta$-fair Bayes-optimal classifier can be expressed as
$$f_{D,\delta}^\star\in \underset{f\in\mathcal{F}_\delta}{\text{argmax}} \int_{\mathcal{A}}\int_{\mathcal{X}}f(x,a)(2\eta_a(x)-1)dP_{X,A}(x,a),$$
with
$$\mathcal{F}_{\delta}=\left\{f:\left|\int_{\mathcal{A}}\int_{\mathcal{X}}f(x,a)\left(\frac{2a-1}{p_a}\right)dP_{X,A}(x,a)\right|\leq \delta\right\}.$$

We first consider the  one-sided constraint:
\begin{equation}\label{constraint1}
   \frac{D^\star}{|D^\star|}\int_{\mathcal{A}}\int_{\mathcal{X}}f(x,a)\left(\frac{2a-1}{p_a}\right)dP_{X,A}(x,a)\leq \delta. 
\end{equation}

Denote $\phi_0(x,a)=2\eta_a(x)-1$ and $\phi_1(x,a)=\frac{(2a-1)D^\star}{p_a|D^\star|}$. 
Moreover, 
let $\mathcal{T}_{\delta}$  be the class of Borel functions $f$ satisfying \eqref{constraint1} 
and
$\mathcal{T}_{\delta,0}$ be the set of functions in $\mathcal{T}_{\delta}$ satisfying \eqref{constraint1} with the inequality replaced by equality. 
It is clear that $\mathcal{T}_{\delta,0}\subset \mathcal{F}_{\delta}\subset\mathcal{T}_{\delta,0}$. 

Now, consider the classifier $f$ of the form
\begin{align}\label{class1}
\begin{split}
f_{s,\tau_1,\tau_0}(x,a)&=\left\{\begin{array}{lcc}
 1,    & & \phi_0(x,a)>s\phi_1(x,a);\\
  \tau_a,   & &\phi_0(x,a)=s\phi_1(x,a);\\
0,     & & \phi_0(x,a)<s\phi_1(x,a).
\end{array}\right.
\end{split}
\end{align}

Note that
$2\eta_a(x)-1>s\frac{(2a-1)D^\star }{p_a|D^\star|}$ is equivalent to
$
\eta_a(x)>\frac{1}{2}+\frac{(2a-1)sD^\star}{2|D^\star|p_a}. 
$
As a result, $f_{s,\tau}(x,a)$ in \eqref{class1} can be written as
\begin{align}\label{opt-in-app}
\begin{split}
 &f_{s,\tau_1,\tau_0}(x,a)=
I\left(\eta_{a}(x)> \frac12+\frac{(2a-1)sD^\star}{2|D^\star|p_a} \right)+ \tau_a I\left(\eta_a(x)=\frac12+\frac{(2a-1)sD^\star}{2|D^\star|p_a}\right).
\end{split}
\end{align}
Moreover,  the constraint \eqref{constraint1} is equivalent to
 \begin{eqnarray}\label{reduce-cons}
    \nonumber &&\frac{D^\star}{|D^\star|}\left[P_{X|A=1}\left({\eta_1}(X)>\frac{1}{2}+\frac{sD^\star}{2|D^\star|p_1}\right)+\tau_1 P_{X|A=1}\left(\eta_1(X)=\frac{1}{2}+\frac{sD^\star}{2|D^\star|p_1}\right)\right.\\
     &&\quad\left.-P_{X|A=0}\left({\eta_0}(X)>\frac{1}{2}-\frac{sD^\star}{2|D^\star|p_0}\right)+\tau_0 P_{X|A=0}\left(\eta_0(X)=\frac{1}{2}-\frac{sD^\star}{2|D^\star|p_0}\right)\right]\leq\delta.
 \end{eqnarray}
 Let $t_{D,\delta}^\star$ be defined as in \eqref{t-star-dp} and $(\tau_{D,\delta,1}^\star,\tau_{D,\delta,0}^\star)$ are given in section \ref{sec-tau1}.
We have
$f_{s_{D,\delta}^\star,\tau_{D,\delta,1}^\star,\tau_{D,\delta,0}^\star}\in \mathcal{T}_{\delta,0}$ with $s_{D,\delta}^\star=\frac{D^\star}{|D^\star|}t_{D,\delta}^\star$. Then, it
follows from the generalized Neyman-Pearson lemma \ref{NP_lemma} that
$$f_{s_{D,\delta}^\star,\tau_{D,\delta}^\star}\in\underset{f\in\mathcal{T}_{\delta,0}}{\text{argmax}}\int_{\mathcal{X}}f_{s,\tau}(x,a)\phi_0(x,a)dP_X(x).$$ 
Recalling the definition of $t_{D,\delta}^\star$ in \eqref{t-star-dp}, we have $t_{D,\delta}^\star\ge0$ when $D^\star>\delta$,
as
$$\textup{DDP}(f_{s,0})
=S_{\eta_1}\left(\frac{1}{2}+\frac{s}{2p_1}\right)-S_{\eta_0}\left(\frac{1}{2}-\frac{s}{2p_1}\right)$$ is monotone non-increasing as a function of $s$, and
$$\textup{DDP}(f_{t_{D,\delta}^\star,0})\leq\textup{DDP}(f_{t_{D,\delta}^\star,\tau_{D,\delta}^\star})=\delta<\textup{DDP}(f_{0,0})=D^\star.$$
Similarly, we have $t_{D,\delta}^\star\le 0$ when $D^\star<-\delta$. It follows that $s_{D,\delta}^\star=\frac{D^\star}{|D^\star|}t_{D,\delta}^\star\ge 0$. As a result,
$$f_{s_{D,\delta}^\star,\tau_{D,\delta}^\star}\in\underset{f\in\mathcal{T}_{\delta}}{\text{argmax}}\int_{\mathcal{X}}f_{s,\tau}(x,a)M(x,a)dP_X(x).$$ 
Since $f_{s_{D,\delta}^\star,\tau_{D,\delta}^\star}\in\mathcal{T}_{\delta,0}\subset \mathcal{F}_{\delta}\subset\mathcal{T}_{\delta,0}$, we can conclude that
$$f_{s_{D,\delta}^\star,\tau_{D,\delta}^\star}\in\underset{f\in\mathcal{F}_{\delta}}{\text{argmax}}\int_{\mathcal{X}}f_{s,\tau}(x,a)M(x,a)dP_X(x).$$ 
The proof is completed by taking $s_{D,\delta}^\star=\frac{D^\star}{|D^\star|}t_{D,\delta}^\star$ in \eqref{opt-in-app}.

\end{proof}


\section{Fair Classification with Cost-Sensitive risk}\label{sec:thm_cost_sensitive}

In Sections \ref{sec:thm_cost_sensitive} to \ref{sec:thm}, 2e extend our theory to fair cost-sensitive classification, multi-class protected attributes, and other group fairness measures, including equality of opportunity, predictive equality and overall accuracy equality. In the rest of the appendix,  we will assume that for all $a\in\mathcal{A}$, $\eta_a(X)$ has a density function on $\mathcal{X}$ to avoid tedious discussions of boundary cases, akin to what we provided in Section \ref{sec-tau1}. 
With this assumption,  we only need to focus on the deterministic classifiers.

We first extend our result to classification with a cost-sensitive risk. 
In many applications, such as detecting email spam, 
predicting recidivism, or diagnosing a medical condition, one of false negatives and false positives can be more harmful than the others. 
In these cases, it helpful to use a cost-sensitive risk, 
taking the costs of prediction errors into account. 
For
a cost parameter $c\in[0,1]$\footnote{When $c=1/2$, cost-sensitive risk reduces to the usual zero-one risk.}, the cost-sensitive 0-1 risk of the classifier $f$ is defined as
$$R_c(f)=c\cdot P(\hat{Y}_f=1,Y=0)+(1-c)\cdot P(\hat{Y}_f=0, Y=1).$$
An unconstrained Bayes-optimal classifier for the cost-sensitive risk is any minimizer $f^\star\in \text{argmin}_f R_c(f).$  
A classical result is that all Bayes-optimal classifiers have the form 
$f^\star(x,a)=I(\eta_a(x)>c)+\tau I(\eta_a(x)=c)$, where $\tau\in[0,1]$ is arbitrary \cite{C2001CS, MW2018}. 
Taking demographic parity into account, the $\delta$-fair Bayes-optimal classifier for a cost-sensitive risk is defined as 

$$f_{D,c,\delta}^\star\in \underset{f\in\mathcal{F}_\delta}{\text{argmin}} R_c(f).$$
where  $\mathcal{F}_{D,\delta}=\{f:|\text{DDP}(f)|\leq \delta\}$ is the set of  measurable functions satisfying the $\delta$-
tolerance constraint.

\begin{theorem}[Fair Bayes-optimal Classifiers with Cost-sensitive Risk]\label{thm-opt-cost-sen}
All fair Bayes-optimal classifiers $f_{D,c,\delta}^\star$ under the  constraint $|\textup{DDP}(f)|\le\delta$ are given as follows.
If  $|D^\star|\leq \delta$,
  then  $f^\star_{D,\delta}$ can be any Bayes-optimal classifier $f^\star$ from Proposition \ref{prop-ba-op}.
Otherwise, for all $x\in \mathcal{X}$ and $a\in \mathcal{A}$,
\begin{align*}
\begin{split}
 &f^\star_{D,\delta}(x,a)=
I\left(\eta_{a}(x)> c+\frac{(2a-1)t^\star_{D,c,\delta}}{p_a}\right),
\end{split}
\end{align*}
where 
$t_{D,c,\delta}^\star$ is defined as
\begin{align*}
\begin{split}
&\sup\left\{t:P_{X|A=1}\left(\eta_1(X)>c+\frac{t}{p_1}\right)>P_{X|A=0}\left(\eta_0(X)>c-\frac{t}{p_0}\right)+\frac{D^\star}{|D^\star|} \delta\right\}.
\end{split}
\end{align*}
\end{theorem}
\begin{proof}
We only sketch the proof, as it follows the same argument as in the proof of Theorem \ref{op-de-rule-dp}.
Note that minimizing the cost-sensitive risk with respect to $f$ is equivalent to maximizing  $(1-c)P(\widehat{Y}_f=1,Y=1)-cP(\widehat{Y}_f=1,Y=0)$, which can be expressed as
\begin{align*}
\begin{split}
&\E\left[(1-c)I(\widehat{Y}_f=1,Y=1)-cI(\widehat{Y}_f=1,Y=0)\right]\\
=&\E\left\{\E\left[(1-c)\left.I(\widehat{Y}_f=1,Y=1)-cI(\widehat{Y}_f=1,Y=0)\right|X,A\right]\right\}\\
=&\int_{\mathcal{A}}\int_{\mathcal{X}}cf(x,a)\eta_a(x)-(1-c)f(x,a)(1-\eta(x,a))dP_{X,A}(x,a)\\
=&\int_{\mathcal{A}}\int_{\mathcal{X}}f(x,a)(\eta_a(x)-c)dP_{X,A}(x,a).
\end{split}
\end{align*}
 Moreover, we have $\textup{DDP}(f) = \int_{\mathcal{A}}\int_{\mathcal{X}}f(x,a)\left(\frac{2a-1}{p_a}\right)dP_{X,A}(x,a).$  When $|D^\star|\leq\delta$, the result is clear. We now consider the case $|D^\star|>\delta$. By the generalized Neyman-Person Lemma \ref{NP_lemma}, the fair Bayes-optimal classifiers take the form
\begin{align*}
\begin{split}
f_{t,\tau_1,\tau_0,c}(x,a)&=\left\{\begin{array}{lcc}
 1,    & & \eta_a(x)-c>\frac{2a-1}{p_a};\\
  \tau_a,   & & \eta_a(x)-c=\frac{2a-1}{p_a};\\
0,     & & \eta_a(x)-c<\frac{2a-1}{p_a}.
\end{array}\right.
\end{split}
\end{align*}
 We can set $\tau_a=0$ as $P_{X|A=a}(\eta_a(X)=t)=0$ for all $a\in\mathcal{A}$ and $t\in[0,1]$, which completes the proof.

\end{proof}

\section{Fair Bayes-optimal Classifier and FairBayes Algorithm with a Multi-class Protected Attribute}\label{sec:thm_multi}

\subsection{Fair Bayes-optimal Classifier with a Multi-class Protected Attribute}
In this section, we extend our theory to the case of a multi-class protected attribute. 
We assume that $A\in\mathcal{A}=\{1,2,...,|\A|\}$ for some integer $|\A|>2$. 
 For simplicity, we only consider perfect demographic parity. Moreover, we assume that 
for $a\in\A$, $\eta_a(X)$ has density on $\mathcal{X}$ to avoid tedious discussions of boundary cases. 
A classifier satisfies demographic parity if
   $$P(\widehat{Y}_f  = 1|A = a)=P(\widehat{Y}_f  = 1),  \ \ \ \text{ for } \ \ a=1,2,...,\A.$$
Denote by $\mathcal{F}_{D}$ the set of measurable functions satisfying demographic parity:
$$\mathcal{F}_{D}=\left\{f:\X\times\A\mapsto[0,1]: \sum_{a=1}^{|\A|}|P(\widehat{Y}_f  = 1|A = a)-P(\widehat{Y}_f  = 1)|= 0\right\}.$$
Then, the  fair Bayes-optimal classifier  is defined as
$$f_{D}^\star\in \underset{f\in\mathcal{F}_D}{\text{argmin}} [P(Y\neq \widehat{Y}_f)].$$

\begin{theorem}\label{thm-opt-dp-multi}
The fair Bayes-optimal classifier $f^\star_{D}$ has the following form:
\begin{align}\label{op-de-rule-dp-multi}
\begin{split}
 &f^\star_{D}(x,a)=
    I\left(\eta_{a}(x)> \frac12+\frac{t^\star_{D,a}}{2p_a}\right).
\end{split}
\end{align}
Here 
$\{t_{D,a}^\star\}_{a=1}^{|\mathcal{A}|}$ satisfy 
$\sum\limits_{a=1}^{|\A|} t_{D,a}^\star=0$, and
for all $a \in \{2,\ldots,|\A|\}$,
\begin{align}\label{t-star-dp-multi}
P_{X|A=a}\left(\eta_a(X)>\frac12+\frac{t_{D,a}^\star}{2p_a}\right)=
P_{X|A=1}\left(\eta_{1}(x)>\frac12+\frac{t_{D,1}^\star}{2p_{1}}\right).
\end{align}

\end{theorem}
\begin{remark}\label{difficulty}
For a multi-class protected attribute, at the moment, our theoretical framework can only handle perfect fairness. 
Under perfect fairness, we have $|\A|-1$ equality constraints, and the fair Bayes-optimal classifier can be derived using the Neyman-Pearson lemma.
However, for approximate fairness for a multi-class protected attribute, the total number of equality and inequality constraints is unknown ahead of time. 
A careful analysis of these two types of constraints is required, and we leave this to future work.
\end{remark}

\begin{proof}
We first demonstrate the existence of $\{t_{D,a}^\star\}_{a=1}^{|\mathcal{A}|}$, in the following lemma.

\begin{lemma}\label{existence-multi}
Suppose that, for all $a\in\mathcal{A}$, $\eta_a(X)$ has a density function on $\mathcal{X}$. Then, there exist  
$\{t_{D,a}^\star\}_{a=1}^{|\mathcal{A}|}$ such that 
$\sum\limits_{a=1}^{|\A|} t_{D,a}^\star=0$, and
for all $a \in \{2,\ldots,|\A|\}$,
\begin{align*}
P_{X|A=a}\left(\eta_a(X)>\frac12+\frac{t_{D,a}^\star}{2p_a}\right)=
P_{X|A=1}\left(\eta_{1}(x)>\frac12+\frac{t_{D,1}^\star}{2p_{1}}\right).
\end{align*}

\end{lemma}

\begin{proof}
Define the functions $\underline{Q}_a, \overline{Q}_a:[0,1]\to\R$,
\begin{align*}
\begin{split}
    \underline{Q}_a(s)&=\sup\left\{t: P_{X|A=a}\left(\eta_a(X)>\frac12+\frac{t}{p_a}\right)>s\right\};  \\
\overline{Q}_a(s)&=\sup\left\{t:P_{X|A=a}\left(\eta_a(X)>\frac12+\frac{t}{p_a}\right)\ge s\right\}.
\end{split}\end{align*} 
By definition, we have the following:
\begin{compactitem}
\item Both $\underline{Q}_a$ and $\overline{Q}_a$ are strictly monotonically decreasing, as $\eta_a(X)$ has a density function;
\item $\underline{Q}_a$ is right continuous and $\underline{Q}_a$ is left continuous;
\item For all $s_0\in[0,1]$, $\lim_{s\to s_0^-}\underline{Q}_a(s)=\lim_{s\to s_0^-}\underline{Q}_a(s)$ and $\lim_{s\to s_0^+}\underline{Q}_a(s)=\lim_{s\to s_0^+}\underline{Q}_a(s)$ with $\underline{Q}_a(s_0)\leq \overline{Q}_a(s_0)$;
\item   For all $t\in[\underline{Q}_a(s),\overline{Q}_a(s)]$, 
$P_{X|A=a}(\eta(X)>\frac{1}{2}+\frac{t}{2p_a})=s.$
\end{compactitem}
Now, we set
 $$s^\star=\sup\{s: \sum_{a=1}^{|\A|}\underline{Q}_a(s)>0\}=\sup\{s: \sum_{a=1}^{|\A|}\overline{Q}_a(s)>0\}.$$
  By the  right continuity of $\sum_{a=1}^{|\A|}\underline{Q}_a(s)$ and left continuity of $\sum_{a=1}^{|\A|}\overline{Q}_a(s)$, we have
  $$\sum_{a=1}^{|\A|}\underline{Q}_a(s^\star)\leq 0 \leq \sum_{a=1}^{|\A|}\overline{Q}_a(s^\star).$$
  Letting 
  $$t^\star_{D,a}=\frac{\sum_{a=1}^{|\A|}\overline{Q}_{a}(s^\star)}{\sum_{a=1}^{|\A|}\overline{Q}_{a}(s^\star)-\sum_{a=1}^{|\A|}\underline{Q}_{a}(s^\star)}\underline{Q}_{a}(s^\star)-\frac{\sum_{a=1}^{|\A|}\underline{Q}_{a}(s^\star)}{\sum_{a=1}^{|\A|}\overline{Q}_{a}(s^\star)-\sum_{a=1}^{|\A|}\underline{Q}_{a}(s^\star)}\overline{Q}_{a}(s^\star),$$
  we have $\sum_{a=1}^{|\A|}t^\star_{D,a}=0$.
  Moreover, as $t^\star_{D,a}\in[\underline{Q}_{a}(s),\overline{Q}_{a}(s)]$ for all $a\in\A$, we have, for all $a\in\A$,
  $$P_{X|A=a}\left(\eta_a(X)>\frac{1}{2}+\frac{t^\star_{D,a}}{2p_a}\right)=s^\star.$$
    This completes the proof.
 \end{proof}

Let $f$ be a classifier that assigns $\widehat{Y}_f=1$ with probability $f(x,a)$ when observing $X=x$ and $A=a$. 
Using the same argument as in the proof of Theory \ref{thm-opt-dp}, we can write
\begin{align*}
\begin{split}
P(\widehat{Y}_f=1,Y=1)-P(\widehat{Y}_f=1,Y=0)&=\int_{\mathcal{A}}\int_{\mathcal{X}}f(x,\tilde{a})\phi_0(x,\tilde{a})dP_{X,A}(x,\ta).
\end{split}
\end{align*}
and
\begin{equation}\label{cons-multi}
P_{X|A=a}(\widehat{Y}_f=1)-P_{X|A=1}(\widehat{Y}_f=1)=\int_{\mathcal{A}}\int_{\mathcal{X}}\phi_a(x,\tilde{a})f(x,\ta)dP_{X,A}(x,\tilde{a})
\end{equation}
with
$$\phi_0(x,\tilde{a})=2\eta_{\tilde{a}}(x)-1; \ \ \ \text{ and } \ \  \ \phi_a(x,\ta)=\frac{I(\ta=a)}{p_a}-\frac{I(\ta=1)}{p_1}.$$
This shows that the fair Bayes-optimal classifier can be expressed as
$$f_{D}^\star\in \underset{f\in\mathcal{F}}{\text{argmax}} \int_{\mathcal{A}}\int_{\mathcal{X}}f(x,\ta)(2\eta_{\ta}(x))dP_{X,A}(x,\ta),$$
with
$$\mathcal{F}=\left\{f:\int_{\mathcal{A}}\int_{\mathcal{X}}f(x,\ta)\left(\frac{I(\ta=a)}{p_a}-\frac{I(\ta=1)}{p_1}\right)dP_{X,A}(x,\ta)=0, \ \ a=2,3,...,|\A|\right\}.$$


For $t = (t_2,\ldots, t_{|\A|})^\top \in \R^{|\A|-1}$ and $\tau = (\tau_2,\ldots, \tau_{|\A|})^\top \in[0,1]^K$, 
let $f$ be the classifier of the form
\begin{align}\label{optimal-mul}
\begin{split}
f(x,a,t_2,\ldots,t_\A)&=\left\{\begin{array}{lcc}
 1,    & & \phi_0(x,\ta)>\sum\limits_{a=2}^{|\mathcal{A}|}t_a\phi_a(x,\ta);\\
 \tau_{\ta},   & & \phi_0(x,\ta)=\sum\limits_{a=2}^{|\mathcal{A}|}t_a\phi_a(x,\ta);\\
0,     & & \phi_0(x,\ta)<\sum\limits_{a=2}^{|\mathcal{A}|}t_a\phi_a(x,\ta).
\end{array}\right.
\end{split}
\end{align}
Writing $t_1=-\sum\limits_{a=2}^{|\mathcal{A}|}t_a$, we have that $\phi_0(x,\ta)>\sum\limits_{a=2}^{|\mathcal{A}|}t_a\phi_a(x,\ta)$ for $\ta\in\A$  is equivalent to
$ \eta_{a}(x)>\frac12+\frac{t_a}{2p_a}$ 
for $a\in\A$, 
Since $\eta_{a}(X)$ has a density function on $\mathcal{X}$, we can set $\tau_a=0$ since $P(\eta_{a}(X)=\frac12+\frac{t_a}{2p_a}|A=a)=0$ for  $a\in\mathcal{A}$. 
As a result, $f$ in \eqref{optimal-mul} can be written as in \eqref{op-de-rule-dp-multi}, with $t^*_{D,a}$ replaced by $t_{a}$, for all $a\in \A$.
{Moreover, since $\sum_{a=1}^{|\A|} t_a = 0$, the constraint \eqref{cons-multi} is equivalent to \eqref{t-star-dp-multi}, with $t^*_{D,a}$ replaced by $t_{a}$, for all $a\in \A$.}
By lemma \ref{existence-multi}, there are $(t^*_{D,a})_{a=1}^{|\A|}$ such that constraint \eqref{t-star-dp-multi} is satisfied. In addition,
 by the generalized Neyman-Pearson lemma \ref{NP_lemma}, we have
$$f_{D}^\star=f(x,a,t^\star_{D,2},\ldots,t^\star_{D,|\A|})\in \underset{f\in\mathcal{F}}{\text{argmax}} \int_{\mathcal{A}}\int_{\mathcal{X}}f(x,\ta)(2\eta_{\ta}(x))dP_{X,A}(x,\ta).$$
which completes the proof.
\end{proof}

\subsection{FairBayes Algorithm for a Multi-class Protected Attribute}

In this section, we propose Algorithm \ref{alg:multi} for fair Bayes-optimal classification with a multi-class protected attribute. Our goal is to find parameters $\{\hat{t}_{a}\}_{a=1}^{|\A|}$ to satisfy the empirical version of the constraint \eqref{t-star-dp-multi}. We adopt the following two-stage method.

In the first step we learn the feature- and group-conditional label probabilities $\eta$ based on the whole dataset,  as in the FairBayes for a binary protected attribute (Algorithm \ref{alg:FBC1}). In the second step,
we divide the data into $|\mathcal{A}|$ parts,
according to the value of $A$: for $a\in\mathcal{A}$, 
$S_a=\{x^{(a)}_{j},y^{(a)}_{j}\}_{j=1}^{n_a}$, where $a^{(a)}_{j}=a$. 
Based on Theorem~\ref{thm-opt-dp-multi}, we consider the following  deterministic classifiers: 
\begin{equation}\label{classifier1}
\tilde{f}(x,a;t_1,t_2,...,t_{|\mathcal{A}|})=I\left(\widehat\eta_{a}(x)> \frac12+\frac{nt_a}{2n_a}\right),
\end{equation}
where ${\eta}$ and $p_a$ are replaced by their empirical estimates, and $t_a$, $a\in\mathcal{A}$, are parameters to learn. 

We use the following strategy to estimate $t_a$, $a\in\A$: First, we fix the first parameter for the group with $a=1$, say $t$.
To achieve demographic parity, we need to find parameters $\{t_a\}_{a=2}^{|\A|}$ for the other groups such that the proportion of positive prediction are the same\footnote{Since a sample mean $n^{-1}\sum_{i=1}^n Z_i$ of iid random variables $Z_i$ has a variability of order $O_P(n^{-1/2})$, even if the true predictive parities are equal, the empirical versions may differ by $O_P(n^{-1/2})$. However, in our case we simply find the values $t_a,t$ for which they are as close as possible.}, i.e., find $t_a$, $a=2,3,\ldots,|\A|$, such that
\begin{equation}\label{thresh-estimate}
\frac1{n_a}\sum_{j=1}^{n_a}I\left(\hat{\eta}_a(x_j^{(a)})>\frac12+\frac{nt_a}{n_a}\right)
\approx
\frac1{n_1}\sum_{j=1}^{n_1}I\left(\hat{\eta}_1(x_j^{(1)})>\frac12+\frac{nt}{n_1}\right)
, \ \ \ \text{ for } a = 2,3,...,|\mathcal{A}|.
\end{equation}
Note that for $a\in\mathcal{A}$, $\frac1{n_a}\sum_{j=1}^{n_a}I\left(\hat{\eta}_a(x_j^{(a)})>\frac12+\frac{nt_a}{n_a}\right)$ is non-increasing as the parameter $t_a$ increases. 
As a consequence, we can search over
$t_a$, $a=2,3,\ldots,|\A|$, 
efficiently  via, for instance, the bisection method. 
We denote by $\hat t_a(t)$, $a=2,3,\ldots,|\A|$, the estimated thresholds given by \eqref{thresh-estimate}, writing $\hat t_1(t)=t$ for convenience. 
We consider the classifier \eqref{classifier1} with these thresholds:
\begin{equation*}
\hat{f}_t(x,a)
=\tilde{f}\left(x,a;\hat t_1(t),\hat t_2(t),...,\hat t_{|\mathcal{A}|}(t)\right)
=I\left(\widehat\eta_{a}(x) > \frac12+\frac{n\hat t_a(t)}{n_a}\right).
\end{equation*}
 Last, we find $t$ such that
\begin{equation}\label{t-estimate}\sum_{a=1}^{|\A|}\hat{t}_a(t)\approx 0.
\end{equation} 
Again, we can search over $t$ efficiently via, for instance, the bisection method
as $\sum_{a=1}^{|\A|}\hat{t}_a(t)$ is non-decreasing in $t$. 
We denote by $\hat {t}$ the parameter satisfying \eqref{t-estimate}.
Our final estimator of the fair Bayes-optimal classifier is $\hat{f}_{D} = \hat f_{\hat t}$. 
\begin{algorithm}[tb]
   \caption{FairBayes for a multi-class protected attribute}
   \label{alg:multi}
\begin{algorithmic}
   \STATE {\bfseries Input:}  Datasets $S=\cup_{a=1}^{|\mathcal{A}|}S_a$ with $S=\{x_{i},a_i,y_{i}\}_{i=1}^{n}$ and $S_a=\{x^{(a)}_{j},y^{(a)}_{j}\}_{j=1}^{n_a}$. Cost parameter $c\in[0,1]$.

   \STATE {\textbf{Step 1}:} Estimate $\eta_a(x)$ by $ \widehat{\eta}\in \underset{f_\theta\in\mathcal{F}}{\text{argmin}}\frac1{n}\sum_{i=1}^{n}L(y_{i},f_{\theta}(x_{i},a_i)).$
           \STATE {\textbf{Step 2}: Find the optimal thresholds.}.
  \FOR{$t \in \mathcal{T}=[0, \max_j\hat\eta_1(x^{(1)}_{j})]$}
    \FOR{$a\in\mathcal{A} \setminus\{1\}$}
    \STATE Find $\hat{t}_a(t)$ such that   {$$
\frac1{n_a}\sum_{j=1}^{n_a}I\left(\hat{\eta}_a(x_j^{(a)})>\frac12+\frac{nt_a}{n_a}\right)
\approx
\frac1{n_1}\sum_{j=1}^{n_1}I\left(\hat{\eta}_1(x_j^{(1)})>\frac12+\frac{nt}{n_1}\right).
$$}
    \ENDFOR  
    \ENDFOR 
   \STATE Find $\hat{t}$ such that  $\sum_{a=1}^{|\A|}\hat{t}_a(\hat{t})\approx 0$. 

\STATE {\bfseries Output:}  
$\widehat{f}_{D}(x,a)=I(\widehat\eta_{a}(x)>\frac{1}{2}+\frac{n\hat{t}_a(\hat{t})}{n_a})$.
    \end{algorithmic}
\end{algorithm}

\section{Bayes-optimal Classifiers under Other Group Fairness Measures}\label{sec:thm}
The following metrics
are widely used to analyze accuracy in classification: the true-positive rate (TPR), true-negative rate (TNR), false-positive rate (FPR)  and false-negative rate (FNR). They are defined as
\begin{align}
\begin{split}
\text{TPR}&=1-\text{FNR}=P(\widehat{Y}_f=1|Y=1);\\
\text{TNR}&=1-\text{FPR}=P(\widehat{Y}_f=0|Y=0).
\end{split} 
\end{align}
Based on these metrics, various group fairness measures can be defined.
 \begin{definition}[Equality of Opportunity \citep{HPS2016}]
    A classifier satisfies equality of opportunity if it achieves the same TPR (or FNR) among  protected groups:
   $$P_{X|A=1,Y=1}(\widehat{Y}_f  = 1) =P_{X|A=0,Y=1}(\widehat{Y}_f  = 1).$$ 
  \end{definition}
  
   \begin{definition}[Predictive Equality \citep{CSFG2017}]
    A classifier satisfies predictive equality  if it achieves the same TNR (or FPR) among protected groups:
   $$P_{X|A=1,Y=0}(\widehat{Y}_f  = 1) =P_{X|A=0,Y=0}(\widehat{Y}_f  = 1 ).$$ 
  \end{definition}

     \begin{definition}[Overall Accuracy Equality \citep{BHJR2021}]
    A classifier satisfies  overall accuracy equality (or equalizing disincentives) if it achieves the same accuracy among different protected groups:
   \begin{align*}&P_{X|A=1,Y=1}(\widehat{Y}  = 1)+P_{X|A=1,Y=0}(\widehat{Y}  = 0) =P_{X|A=0,Y=1}(\widehat{Y}  = 1)+P_{X|A=0,Y=0}(\widehat{Y}  = 0).
   \end{align*}
  \end{definition}

Our theory and FairBayes algorithm can be naturally extended to these group fairness measures. 
We now introduce the fair Bayes-optimal classifiers and our variants of the FairBayes algorithm for these measures. 
In our setting, for a given fairness measure, the level of disparity is quantified by the difference between the probabilities equalized in this fairness measure.  As a result, we use DEO, DPE and DOA to measure the degree of violating
equality of opportunity, predictive equality and overall accuracy equality, respectively:
\begin{align*}
    \begin{split}
  \text{DEO}&= P_{X|A = 1,Y = 1}(\widehat{Y}_f = 1) -P_{X|A = 0,Y=1}(\widehat{Y}_f  = 1);\\
   \text{DPE}&= P_{X|A = 1,Y=0}(\widehat{Y}_f  = 1) -P_{X|A = 0,Y=0}(\widehat{Y}_f  = 1);\\
   \text{DOA}&= P_{X|A = 1,Y=1}(\widehat{Y}_f  = 1)+ P_{X|A = 1,Y=0}(\widehat{Y}_f  = 0)\\
   &\quad- P_{X|A = 0,Y=1}(\widehat{Y}_f  = 1)- P_{X|A = 0,Y=0}(\widehat{Y}_f  = 0).
    \end{split}
\end{align*}

\subsection{Bayes-Optimal Classifiers under Group Fairness}
In this section, we introduce the fair Bayes-optimal classifiers under other group fairness measure.  We define the following functional sets  satisfying the \emph{$\delta$- tolerance} constraint under equality of opportunity, predictive equality and overall accuracy equality, respectively.
\begin{align*}
    \begin{split}
  \mathcal{F}_{E,\delta}=\{f: |\textup{DEO}(f)|\le \delta\};\qquad
  \mathcal{F}_{P,\delta}=\{f: |\textup{DPE}(f)|\le \delta\};\qquad
    \mathcal{F}_{O,\delta}=\{f: |\textup{DOA}(f)|\le \delta\}.
    \end{split}
\end{align*}

The $\delta$-fair Bayes-optimal classifier with respect to equality of opportunity, predictive equality and overall accuracy equality are respectively  defined as
\begin{align*}
    \begin{split}
  f_{E,\delta}^\star\in \underset{f\in\mathcal{F}_\delta}{\text{argmin}} [P(Y\neq \widehat{Y}_f)];\qquad
  f_{P,\delta}^\star\in \underset{f\in\mathcal{F}_\delta}{\text{argmin}} [P(Y\neq \widehat{Y}_f)];\qquad
    f_{O,\delta}^\star\in \underset{f\in\mathcal{F}_\delta}{\text{argmin}} [P(Y\neq \widehat{Y}_f)].
    \end{split}
\end{align*}

Let $E^\star=\sup_{f^*}\textup{DEO}(f^\star)$, $P^\star=\sup_{f^*}\textup{DPE}(f^\star)$ and $O^\star=\sup_{f^*}\textup{DOA}(f^\star)$ with the supremum be taken over all Bayes-optimal classifiers $f^\star$  from Proposition \ref{prop-ba-op}.
{Again, we assume $\eta_1(X)$ and $\eta_0(X)$ have density functions on $\mathcal{X}$ to avoid tedious discussions of boundary cases.}
Our results are summarized in Theorem \ref{thm-opt-EO},Theorem \ref{thm-opt-PE} and Theorem \ref{thm-opt-OA}, respectively.

\begin{theorem}[Fair Bayes-optimal Classifiers under Equality of Opportunity]\label{thm-opt-EO}
Let $E^\star=\textup{DEO}(f^\star)$.
For any $\delta> 0$, all fair Bayes-optimal classifiers $f_{E,\delta}^\star$ under the fairness constraint $|\textup{DEO}(f)|\le\delta$ are given as follows: If $|E^\star|\leq \delta$,
then  $f^\star_{E,\delta}$ can be any Bayes-optimal classifier $f^\star$ from Proposition \ref{prop-ba-op}.
    Otherwise for all $x\in \mathcal{X}$ and $a\in \mathcal{A}$,
\begin{align}\label{opt-de-EO}
\begin{split}
 &f^\star_{E,\delta}(x,a)=
I\left(\eta_{a}(x)> \frac{p_ap_{Y,a}}{2p_ap_{Y,a}+(1-2a)t_{E,\delta}^\star}\right)
\end{split}
\end{align}
where
$t_{E,\delta}^\star$ is defined as
\begin{align*}
\begin{split}
&t_{E,\delta}^\star=\sup\left\{t:P_{X|A=1,Y=1}\left(\eta_1(X)>\frac{p_1p_{Y,1}}{2p_1p_{Y,1}-t}\right)\right.\\
&\qquad\qquad\qquad\qquad\qquad\qquad\left.>P_{X|A=0,Y=1}\left(\eta_0(X)>\frac{p_0p_{Y,0}}{2p_0p_{Y,0}+t}\right)+\frac{E^\star}{|E^\star|} \delta\right\}.
\end{split}
\end{align*}
\end{theorem}

\begin{theorem}[Fair Bayes-optimal Classifiers under Predictive Equality]\label{thm-opt-PE}
Let $P^\star=\textup{DPE}(f^\star)$.
For any $\delta> 0$, all fair Bayes-optimal classifiers $f_{P,\delta}^\star$ under the fairness constraint $|\textup{DPE}(f)|\le\delta$ are given as follows:
If $|P^\star|\leq \delta$, then
$f^\star_{P,\delta}$ can be any Bayes-optimal classifier $f^\star$ from Proposition \ref{prop-ba-op}.
Otherwise,  
for all $x\in \mathcal{X}$ and $a\in \mathcal{A}$,
\begin{align}\label{opt-de-pp}
\begin{split}
 &f^\star_{P,\delta}(x,a)=
I\left(\eta_{a}(x)>\frac{p_a(1-p_{Y,a})+(2a-1)t_{P,\delta}^\star}{2p_a(1-p_{Y,a})+(2a-1)t_{P,\delta}^\star}\right),
\end{split}
\end{align}
where 
$t_{P,\delta}^\star$ is defined as
\begin{align*}
\begin{split}
&t_{P,\delta}^\star=\sup\left\{t:P_{X|A=1,Y=0}\left(\eta_1(X)>\frac{p_1(1-p_{Y,1})+t}{2p_1(1-p_{Y,1})+t}\right)\right.\\
&\qquad\qquad\qquad\qquad\qquad\qquad\left.>P_{X|A=0,Y=0}\left(\eta_0(X)>\frac{p_0(1-p_{Y,0})-t}{2p_0(1-p_{Y,0})-t}\right)+\frac{P^\star}{|P^\star|} \delta\right\}.
\end{split}
\end{align*}
\end{theorem}

\begin{theorem}[Fair Bayes-optimal Classifiers under Overall Accuracy Equality]\label{thm-opt-OA}
Let $O^\star=\textup{DOA}(f^\star)$.
For any $\delta> 0$, all fair Bayes-optimal classifiers $f_{O,\delta}^\star$ under the fairness constraint $|\textup{DOA}(f)|\le\delta$ are given as follows.
If $|O^\star|\leq \delta$, then
  $f^\star_{O,\delta}$ can be any Bayes-optimal classifier $f^\star$ from Proposition \ref{prop-ba-op}.
Otherwise,  
for all $x\in \mathcal{X}$ and $a\in \mathcal{A}$,
\begin{align*}
 f^\star_{O,\delta}(x,a)&=
I\left(\eta_{a}(x)> \zeta_{O}(t^\star_{O,\delta},a)\right)
\end{align*}
where 
$$\zeta_{O}(t,a)=\frac{p_{a}p_{Y,a}(1-p_{Y,a})+(1-2a)P_{Y,a}t_{O,\delta}^\star}{2p_ap_{Y,a}(1-p_{Y,a})+(1-2a)t},$$ 
and
$t_{O,\delta}^\star$ is defined as
\begin{align*}
\begin{split}
&t_{O,\delta}^\star=\sup\left\{t:P_{X|A=1,Y=1}\left(\eta_1(X)>\zeta_{O}(t,1)\right)-P_{X|A=1,Y=0}\left(\eta_1(X)>\zeta_{O}(t,1)\right)\right.\\
&\qquad \qquad > \left.P_{X|A=0,Y=1}\left(\eta_0(X)>\zeta_{O}(t,0)\right)-P_{X|A=0,Y=0}\left(\eta_0(X)>\zeta_{O}(t,0)\right)+\frac{O^\star}{|O^\star|} \delta\right\}.
\end{split}
\end{align*}
\end{theorem}


\subsection{Proofs of Theorem \ref{thm-opt-EO} to Theorem \ref{thm-opt-OA}}\label{sec:proof2}
\begin{proof}
Again, we only sketch the proof as it follows the same argument as in the proof of Theorem \eqref{op-de-rule-dp}. 
First, we note that
$$\eta_a(x)=\frac{p_{Y,a}dP_{X|A=a,Y=1}(x)}{dP_{X|A=a}(x)},$$
and
$$dP_{X|A=a}(x)=p_{Y,a}dP_{X|A=a,Y=1}(x)+(1-p_{Y,a}dP_{X|A=a,Y=0}(x)).$$
It follows that
$$\frac{dP_{X|A=a,Y=1}(x)}{dP_{X|A=a}(x)}=\frac{\eta(x)}{p_{Y,a}}; \ \text{ and } \  \frac{dP_{X|A=a,Y=0}(x)}{dP_{X|A=a}(x)}=\frac{1-\eta(x)}{1-p_{Y,a}}.$$

Let   $f$ be any randomized classifier. Its $\text{DEO}$, $\text{DPE}$ and $\text{DOA}$ can be expressed in turn as
\begin{align*}
\textup{DEO}(f)&=P(\widehat{Y}_f  = 1|A=1,Y=1)-P(\widehat{Y}_f=1|A=0,  = 1)\\
&=\int_{\mathcal{X}}f(x,1)dP_{X|A=1,Y=1}(x)-\int_{\mathcal{X}}f(x,0)dP_{X|A=0,Y=1}(x)\\
&=\sum_{a\in\mathcal{A}}\left(\int_{\mathcal{X}}\left[\frac{I(a=1)}{p_1}f(x,a)-\frac{I(a=0)}{p_0}f(x,a)\right]dP_{X|A=a,Y=1}(x)\right)\\
&=\int_{\mathcal{A}}\int_{\mathcal{X}}f(x,a)\left(\frac{2a-1}{p_a}\right)\frac{dP_{X|A=a,Y=1}(x)}{dP_{X|A=a}}{dP_{X|A=a}}dP_A(a)\\
&=\int_{\mathcal{X}\times\mathcal{A}}f(x,a)\left(\frac{(2a-1)\eta_a(x)}{p_ap_{Y,a}}\right)dP_{X,A}(x,a);\\
\textup{DPE}(f)&=P(\widehat{Y}_f  = 1|A=1,Y=0)-P(\widehat{Y}_f  = 1|A=0,Y=0)\\
&=\int_{\mathcal{X}}f(x,1)dP_{X|A=1,Y=0}(x)-\int_{\mathcal{X}}f(x,0)dP_{X|A=0,Y=0}(x)\\
&=\sum_{a\in\mathcal{A}}\left(\int_{\mathcal{X}}\left[\frac{I(a=1)}{p_1}f(x,a)-\frac{I(a=0)}{p_0}f(x,a)\right]dP_{X|A=a,Y=0}(x)\right)\\
&=\int_{\mathcal{A}}\int_{\mathcal{X}}f(x,a)\left(\frac{2a-1}{p_a}\right)\frac{dP_{X|A=a,Y=0}(x)}{dP_{X|A=a}}{dP_{X|A=a}}dP_A(a)\\
&=\int_{\mathcal{X}\times\mathcal{A}}f(x,a)\left(\frac{(2a-1)(1-\eta_a(x))}{p_a(1-p_{Y,a})}\right)dP_{X,A}(x,a);\\
\textup{DOA}(f)&= P(\widehat{Y}_f  = 1|A = 1,Y=1)+ P(\widehat{Y}_f  = 0|A = 1,Y=0)\\
&\qquad\qquad\qquad\qquad- P(\widehat{Y}_f  = 1|A = 0,Y=1)- P(\widehat{Y}_f  = 0|A = 0,Y=0)\\
&= P(\widehat{Y}_f  = 1|A = 1,Y=1)- P(\widehat{Y}_f  = 1|A = 1,Y=0)\\
&\qquad\qquad\qquad\qquad-  P(\widehat{Y}_f  = 1|A = 0,Y=1)+P(\widehat{Y}_f  = 1|A = 0,Y=0)\\
&=\int_{\mathcal{X}}f(x,1)[dP_{X|A=1,Y=1}(x)-dP_{X|A=1,Y=0}(x)]\\
&\qquad\qquad\qquad\qquad- \int_{\mathcal{X}}f(x,0)[dP_{X|A=0,Y=1}(x)-dP_{X|A=0,Y=0}(x)]\\
&=\int_{\mathcal{A}}\int_{\mathcal{X}}f(x,a)\frac{I(a=1)}{p_1}\frac{dP_{X|A=a,Y=1}(x)-dP_{X|A=1,Y=0}(x)}{dP_{X|A=a}}dP_{X|A=a}dP_A(a)\\
&\quad-\int_{\mathcal{A}}\int_{\mathcal{X}}f(x,a)\frac{I(a=0)}{p_0}\frac{dP_{X|A=a,Y=1}(x)-dP_{X|A=a,Y=0}(x)}{dP_{X|A=a}}dP_{X|A=a}dP_A(a)\\
&=\int_{\mathcal{X}\times\mathcal{A}}f(x,a)\frac{2a-1}{p_a}\left(\frac{\eta_a(x)}{p_{Y,a}}-\frac{1-\eta_a(x)}{1-p_{Y,a}}\right)dP_{X,A}(x,a).
\end{align*}
Denote
\begin{align*}
\phi_E(x,a)&=\frac{(2a-1)\eta_a(x)}{p_ap_{Y,a}};\qquad
\phi_P(x,a)=\frac{(2a-1)(1-\eta_a(x))}{p_a(1-p_{Y,a})};\\
\phi_O(x,a)&=\frac{2a-1}{p_a}\left(\frac{\eta_a(x)}{p_{Y,a}}-\frac{1-\eta_a(x)}{1-p_{Y,a}}\right).
\end{align*}
Again, when $|E^\star|$ (or $|P^\star|$, $|O^\star|$) $\leq\delta$, the result is clear. We now consider the case $|E^\star|$ (or $|P^\star|$, $|O^\star|$) $>\delta$. By the generalized Neyman-Person Lemma \ref{NP_lemma}, the fair Bayes-optimal classifiers take the form
\begin{align*}
\begin{split}
f_{t,\tau_1,\tau_0}(x,a)&=\left\{\begin{array}{lcc}
 1,    & & \phi_0(x,a)>t\phi_1(x,a);\\
  \tau_a,   & & \phi_0(x,a)=t\phi_1(x,a);\\
0,     & & \phi_0(x,a)<t\phi_1(x,a).
\end{array}\right.
\end{split}
\end{align*}
Here $\phi_0(x,a)=2\eta_a(x)-1$ and $\phi_1$ can be $\phi_E$, $\phi_P$ or $\phi_O$, depending the fairness measure. Similar to the proof of Theorem \ref{thm-opt-dp}, we can derive an equivalence between $M(x,a)>tH(x,a)$ and the specified thresholding rules, finishing the proof.

\end{proof}

\subsection{Algorithms Aiming for Group Fairness}
\label{sec:alg}

\begin{algorithm}[!t]
   \caption{FairBayes Algorithm for Other Group Fairness Measures: Equality of Opportunity, Predictive Equality, and Overall Accuracy Equality.}
   \label{alg:Fair-EO}
\begin{algorithmic}

   \STATE {\bfseries Input:} Fairness metric: Equality of Opportunity, Predictive Equality, Overall Accuracy Equality.
   Tolerance level $\delta$. Datasets $S_1=\{x_{1,i},y_{1,i}\}_{i=1}^{n_1}$ and $S_0=\{x_{0,i},y_{0,i}\}_{i=1}^{n_0}$.

   \STATE {\textbf{Step 1}:} Learn $\eta_a$  by
     \qquad  $ \widehat{\eta}\in \underset{f_\theta\in\mathcal{F}}{\text{argmin}}\frac1{n}\sum_{i=1}^{n}L(y_{i},f_{\theta}(x_{i},a_i)).$

       \STATE {\textbf{Step 2}:} Find the optimal thresholds:
  
      \IF{Fairness metric == Equality of Opportunity}
\IF{$|\widehat{E}({0})|\leq \delta$}
   \quad$\widehat{t}_{E,\delta}=0$;
  \ELSIF{$\widehat{E}({0})> \delta$}
  \quad$\widehat{t}_{E,\delta}=\sup\left\{t :\widehat{E}(t)>\delta\right\}$;
     \ELSE
  \quad$\widehat{t}_{E,\delta}=\sup\left\{t :\widehat{E}(t)>-\delta\right\}$;
\ENDIF
\STATE {\bfseries Output:} 
$\widehat{f}_{E,\delta}(x,a)=
I\left(\widehat\eta_{a}(x)>\frac{\widehat{p}_a\widehat{p}_{Y,a}}{2\widehat{p}_a\widehat{p}_{Y,a}+(1-2a)\widehat{t}_{E,\delta}}\right).$
   \vspace{0.3cm}     
   \ELSIF{Fairness metric == Predictive Equality}
\IF{$|\widehat{P}({0})|\leq \delta$}
   \quad$\widehat{t}_{P,\delta}=0$;
  \ELSIF{$\widehat{P}({0})> \delta$}
  \quad$\widehat{t}_{P,\delta}=\sup\left\{t :\widehat{P}(t)>\delta\right\}$;
     \ELSE
  \quad$\widehat{t}_{P,\delta}=\sup\left\{t :\widehat{P}(t)>-\delta\right\}$;
\ENDIF
\STATE {\bfseries Output:} 
$\widehat{f}_{P,\delta}(x,a)=
I\left(\widehat\eta_{a}(x)>\frac{\widehat{p}_a(1-\widehat{p}_{Y,a})+(2a-1)\widehat{t}_{P,\delta}}{2\widehat{p}_a(1-\widehat{p}_{Y,a})+(2a-1)\widehat{t}_{P,\delta}}\right).$
  \vspace{0.3cm}
   \ELSIF{Fairness metric == Overall Accuracy Equality}
\IF{$|\widehat{O}({0})|\leq \delta$}
   \quad$\widehat{t}_{O,\delta}=0$;
  \ELSIF{$\widehat{O}({0})> \delta$}
  \quad$\widehat{t}_{O,\delta}=\sup\left\{t :\widehat{O}(t)>\delta\right\}$;
     \ELSE
  \quad$\widehat{t}_{O,\delta}=\sup\left\{t :\widehat{O}(t)>-\delta\right\}$;
\ENDIF
\STATE {\bfseries Output:} 
$\widehat{f}_{O,\delta}(x,a)=
I\left(\widehat\eta_{a}(x)>\frac{\widehat{p}_{a}\widehat{p}_{Y,a}(1-\widehat{p}_{Y,a})+(1-2a)\widehat{p}_{Y,a}\widehat{t}_{O,\delta}}{2\widehat{p}_a\widehat{p}_{Y,a}(1-\widehat{p}_{Y,a})+(1-2a)\widehat{t}_{O,\delta}}\right).$
\ENDIF
\end{algorithmic}
\end{algorithm}

Next, we introduce algorithms to reduce unfairness measures, see Algorithm \ref{alg:Fair-EO}.
As in the main text, we only consider deterministic classifiers, as our primary goal is to learn the optimal thresholds. 
Suppose we observe independently and identically distributed data points $(x_i,a_i,y_i)_{i=1}^n$. we define the following index sets, for $a,y\in\{0,1\}$: $\mathcal{S}_{a}=\mathcal{S}_{a,1}\cup\mathcal{S}_{a,0}$ with
$\mathcal{S}_{a,y}=\{i: (a_i,y_i)=(a,y)\}$.

For all fairness measures, step 1 is exactly the same as in FairBayes (Algorithm \eqref{alg:FBC1}) for demographic parity. 
In the second step, we consider the (plug-in) group-wise thresholding rule
\begin{align*}
{f}(x,a,t)&=
I\left(\widehat\eta_{a}(x)>t_a \right),
\end{align*}
where $t_a$ are determined by the fairness constraint. 
The unfairness measures $\text{DEO}(f)$, $\text{DPE}(f)$ and $\text{DOA}(f)$ are estimated, respectively,
  as
\begin{align*}
    \widehat{E}(t)&=\sum_{i\in\mathcal{S}_{1,1}}I\left(\widehat{\eta}_1(x_i)>\frac{\widehat{p}_1\widehat{p}_{Y,1}}{2\widehat{p}_1\widehat{p}_{Y,1}-t}
\right)-\sum_{i\in\mathcal{S}_{0,1}}I\left(\widehat{\eta_0}(x_i)>\frac{\widehat{p}_0\widehat{p}_{Y,0}}{2\widehat{p}_0\widehat{p}_{Y,0}+t}\right);\\
\widehat{P}(t)&=\sum_{i\in\mathcal{S}_{1,0}}I\left(\widehat{\eta}_1(x_i)>\frac{\widehat{p}_1(1-\widehat{p}_{Y,1})+t}{2\widehat{p}_1(1-\widehat{p}_{Y,1})+t}
\right)-\sum_{i\in\mathcal{S}_{0,0}}I\left(\widehat{\eta}_0(x_i)>\frac{\widehat{p}_0(1-\widehat{p}_{Y,0})-t}{2\widehat{p}_0(1-\widehat{p}_{Y,0})-t}\right);\\
\widehat{O}(t)&=
\sum_{i\in\mathcal{S}_{1,1}}I\left(\widehat{\eta}_1(x_i)>\frac{\widehat{p}_1\widehat{p}_{Y,1}(1-\widehat{p}_{Y,1})-\widehat{p}_{Y,1}t}{2\widehat{p}_1\widehat{p}_{Y,1}(1-\widehat{p}_{Y,1})-t}
\right)-I\sum_{i\in\mathcal{S}_{1,0}}\left(\widehat{\eta}_1(x_i)>\frac{\widehat{p}_1\widehat{p}_{Y,1}(1-\widehat{p}_{Y,1})-\widehat{p}_{Y,1}t}{2\widehat{p}_1\widehat{p}_{Y,1}(1-\widehat{p}_{Y,1})-t}
\right)\\
&-\sum_{i\in\mathcal{S}_{0,1}}I\left(\widehat{\eta}_0(x_i)>\frac{\widehat{p}_0\widehat{p}_{Y,0}(1-\widehat{p}_{Y,0})+\widehat{p}_{Y,0}t}{2\widehat{p}_0\widehat{p}_{Y,0}(1-\widehat{p}_{Y,0})+t}\right)+\sum_{i\in\mathcal{S}_{0,0}}I\left(\widehat{\eta}_0(x_i)>\frac{\widehat{p}_0\widehat{p}_{Y,0}(1-\widehat{p}_{Y,0})+\widehat{p}_{Y,0}t}{2\widehat{p}_0\widehat{p}_{Y,0}(1-\widehat{p}_{Y,0})+t}\right).
\end{align*}
Here $\widehat{p}_a$ and $\widehat{p}_{Y,a}$ are plug-in estimators for $p_a$ and $p_{Y,a}$, respectively. When $|\widehat{E}(0)|\leq \delta$, (or the same inequality holds for $|\widehat{P}(0)|$, $|\widehat{O}(0)|$), we set
$\widehat{t}_{E,\delta}=0$ (or, do the same for  $\widehat{t}_{P,\delta}$, $\widehat{t}_{O,\delta}$).
However, if $|\widehat{E}(0)|> \delta$ (or the same inequality holds for $|\widehat{P}(0)|$, $|\widehat{O}(0)|$),  we set
\begin{align}\label{empirical_t_EO}
\begin{split}
  \widehat{t}_{E,\delta}&=\sup\left\{t :\widehat{E}(t)>\frac{\widehat{E}({0})}{|\widehat{E}({0})|}\delta\right\};\quad
    \widehat{t}_{P,\delta}=\sup\left\{t :\widehat{P}(t)>\frac{\widehat{P}({0})}{|\widehat{P}({0})|}\delta\right\};\\
      \widehat{t}_{O,\delta}&=\sup\left\{t :\widehat{O}(t)>\frac{\widehat{O}({0})}{|\widehat{O}({0})|}\delta\right\}.
   \end{split}
\end{align}

The final estimators of the fair Bayes-optimal classifiers under the above group fairness measures are

\begin{align*}
\widehat{f}_{E,\delta}(x,a)&=
I\left(\widehat\eta_{a}(x)>\frac{\widehat{p}_a\widehat{p}_{Y,a}}{2\widehat{p}_a\widehat{p}_{Y,a}+(1-2a)\widehat{t}_{E,\delta}}\right);\\
\widehat{f}_{P,\delta}(x,a)&=
I\left(\widehat\eta_{a}(x)>\frac{\widehat{p}_a(1-\widehat{p}_{Y,a})+(2a-1)\widehat{t}_{E,\delta}}{2\widehat{p}_a(1-\widehat{p}_{Y,a})+(2a-1)\widehat{t}_{E,\delta}}\right);\\
\widehat{f}_{O,\delta}(x,a)&=
I\left(\widehat\eta_{a}(x)>\frac{\widehat{p}_{a}\widehat{p}_{Y,a}(1-\widehat{p}_{Y,a})+(1-2a)\widehat{p}_{Y,a}\widehat{t}_{O,\delta}}{2\widehat{p}_a\widehat{p}_{Y,a}(1-\widehat{p}_{Y,a})+(1-2a)\widehat{t}_{O,\delta}}\right).
\end{align*}


\section{More Experimental Details and Results}\label{sec:more_experiments}
\subsection{Synthetic Data}
For the Gaussian synthetic data, we consider two combinations of $(p,\sigma)$. The result for $p=10$ and $\sigma=1$ is shown in the main text. Here, we present the result for $p=2$ and $\sigma=0.5$. We employ the same training settings and notations as described in the main text. The results are presented in Table \ref{table_ddp_dim_2}.
We observe a similar pattern as we did in the main text. The proposed FairBayes algorithm closely tracks the  behavior of the fair Bayes-optimal classifier.

\begin{table}[!h]
\caption{Classification accuracies and levels of disparity for the fair Bayes-optimal classifier and FairBayes using logistic regression.}
\label{table_ddp_dim_2}
\vskip -0.2in
\begin{center}
\setlength{\tabcolsep}{5.4pt}
\renewcommand{\arraystretch}{0.95}
\begin{small}
\begin{sc}
\begin{tabular}{cc|cc|cc|cc}
\hline
\multicolumn{4}{c|}{Demographic Parity}  & \multicolumn{4}{c}{Equality of Opportunity} \\\hline
\multicolumn{2}{c|}{Theoretical}  & \multicolumn{2}{c|}{FairBayes} &
\multicolumn{2}{c|}{Theoretical}& \multicolumn{2}{c}{FairBayes} \\ \hline
$\delta$& $\text{ACC}_{D,\delta}$ & DDP & $\text{ACC}_{D,\delta}$ &  $\delta$& $\text{ACC}_{E,\delta}$& DEO& $\text{ACC}_{E,\delta}$ \\ 
\hline
0.00 & 0.766 & 0.016(0.012) & 0.766(0.007) & 0.00 & 0.789  & 0.015 (0.010) & 0.789 (0.006) \\
0.10 & 0.783 & 0.101(0.019) & 0.782(0.007) & 0.06 & 0.798 & 0.062 (0.019) & 0.797 (0.006) \\ 
0.20 & 0.795 & 0.200(0.020) & 0.795(0.006) & 0.12 & 0.803 & 0.121 (0.020)& 0.802 (0.006)\\
0.30 & 0.803 & 0.300(0.018) & 0.803(0.006) &0.18 & 0.805 & 0.181(0.021) & 0.805(0.006)\\\hline\end{tabular}
\end{sc}
\end{small}
\end{center}
\vskip -0.15in
\end{table}

\subsection{Empirical Data Analysis}
{\bf Training details.} We employ the same settings as in \cite{CHS2020} for all algorithms, except for the batch size of our FairBayes method. As our FairBayes method splits the data into two parts, we also divide the batch size by two accordingly. The details are summarized in Table \ref{detail}.

\begin{table}[!h]
\caption{Training details on three datasets}
\label{detail}
\begin{center}
\begin{small}
\begin{sc}
\begin{tabular}{c|c|c|c}
\hline
  Dataset &Adult &  COMPAS& Law School \\
\hline
 Batch size  &512 &2048& 2048\\
Training Epochs &200 & 500  &200 \\
Learning rate & 1e-1
&5e-4&   2e-4 \\\hline
\end{tabular}
\end{sc}
\end{small}
\end{center}
\end{table}

{\bf Performance on  the ``COMPAS'' and ``LawSchool''  datasets.} 

\begin{compactitem}

\item {\it COMPAS:} In  the COMPAS dataset, $Y$  indicates whether or not a criminal will reoffend. Here $X$ includes  prior criminal records, age and an indicator of misdemeanor.  The protected attribute $A$  is the race of an individual, ``white-vs-non-white''.

\item {\it{Law school}}: The target variable $Y$ in the Law school dataset is whether a student gets admitted to law school. 
Thus $X$ includes the LSAT scores, undergraduate GPA and more. The protected attribute $A$ we consider is the race,``white-vs-non-white''.
\end{compactitem}

Tables  \ref{table_COMPAS} and \ref{table_LawSchool} present the experimental results on the ``COMPAS'' and ``LawSchool'' datasets, respectively. We observe that FairBayes
 outperforms other methods on the Lawschool dataset with   perfect disparity control and satisfactory model accuracy. The advantage of  is further supported by Figure \ref{fig:Law}, where we see that FairBayes achieves the best fairness-accuracy tradeoff.

\begin{table}[b]
 \caption{Classification accuracy and level of unfairness on the COMPAS dataset (Here, domain based algorithms and PPUO are designed for demographic parity).}
 \label{table_COMPAS}
 \vskip -0.1in
 \begin{center}
 \begin{small}
 \begin{sc}
 \begin{tabular}{c|c|cc|cc}\hline
  \multicolumn{2}{c|}{}  &\multicolumn{2}{c|}{Demographic Parity}  & \multicolumn{2}{c}{Equality of Opportunity} \\\hline
 method &   parameter &ACC$_D$ & DDP & ACC$_E$& DEO\\\hline
 FairBayes & $\delta=0 $&0.651 (0.004) & 0.007 (0.007)& 0.647 (0.005)& 0.029 (0.013)\\
KDE based &$\lambda=0.75$&0.647 (0.003) &	0.006 (0.004)&	 0.643 (0.004)&	0.023 (0.010)\\
Adversarial&$\alpha=3$& 0.593 (0.056) &	 0.14 (0.078) &0.624 (0.029)&	 0.156 (0.052)\\
PPUO && 0.652 (0.005)     & 0.016 (0.011)    &      &   \\
 Domain dis & & 0.663(0.003)	&0.208(0.005) &&	\\
 Domain ind& & 0.672 (0.003) &	0.217 (0.005)& &	\\
\hline
\end{tabular}
 \end{sc}
\end{small}
 \end{center}
 \vskip -0.2in
 \end{table}
 
\begin{table}[!t]
 \caption{Classification accuracy and level of unfairness on the LawSchool dataset (Here, domain based algorithms and PPUO are designed for demographic parity).}
 \label{table_LawSchool}
 \vskip -0.1in
 \begin{center}
 \begin{small}
 \begin{sc}
 \begin{tabular}{c|c|cc|cc}\hline
  \multicolumn{2}{c|}{}  &\multicolumn{2}{c|}{Demographic Parity}  & \multicolumn{2}{c}{Equality of Opportunity} \\\hline
 method &   parameter &ACC$_D$ & DDP & ACC$_E$& DEO\\\hline
 FairBayes & $\delta=0 $&0.787 (0.000) & 0.001 (0.001)& 0.788 (0.000)& 0.011 (0.005)\\
KDE based &$\lambda=0.75$&0.788 (0.001) &	0.021 (0.002)&	 0.788 (0.000)&	0.028 (0.004)\\
Adversarial&$\alpha=3$& 0.777 (0.005) &	 0.057 (0.012) &0.746 (0.004)&	 0.044(0.016)\\
PPUO && 0.790 (0.000)     & 0.070 (0.001)    &      &   \\
 Domain dis & & 0.790(0.001)	&0.097(0.002) &&	\\
 Domain ind& & 0.788 (0.001) &	0.124 (0.002)& &	\\
\hline
\end{tabular}
 \end{sc}
\end{small}
 \end{center}
 \vskip -0.2in
 \end{table}

\begin{figure}[!t]
\begin{center}
\centerline{\includegraphics[width=0.95\columnwidth]{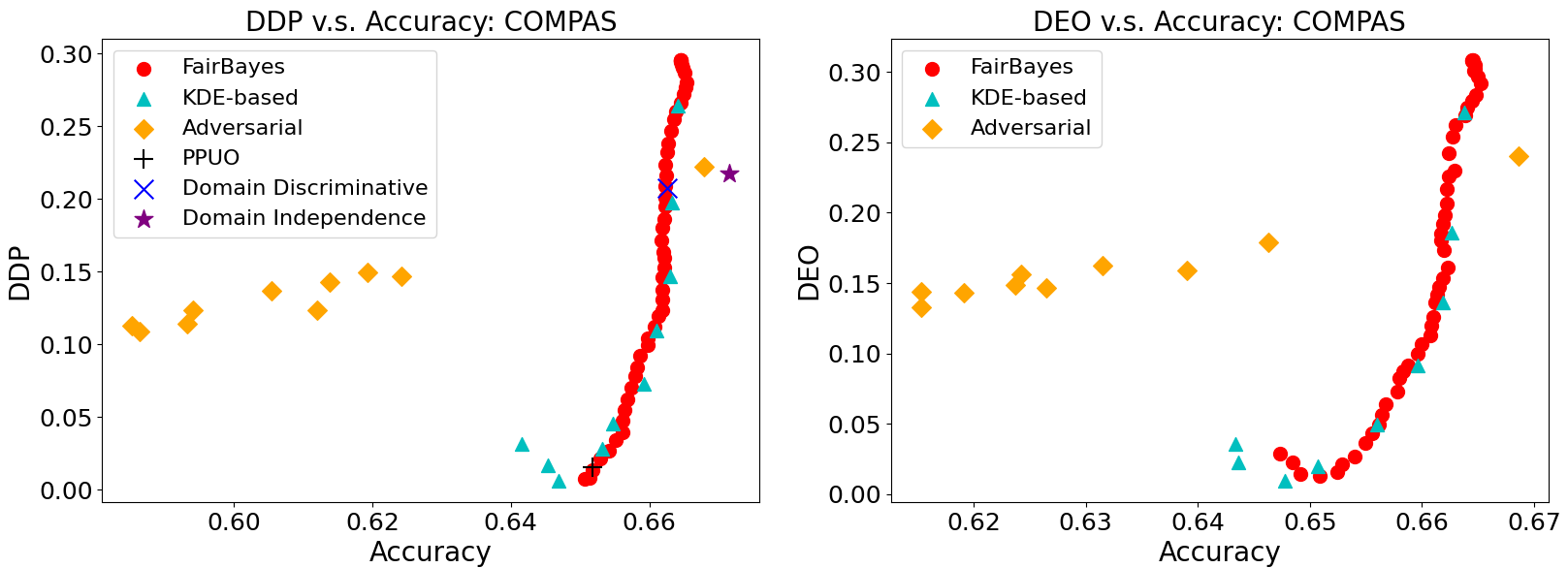}}
\caption{Fairness-Accuracy tradeoff on the ``COMPAS'' dataset. Left panel: Tradeoff with respect to demographic parity (DDP). Right panel: Tradeoff with respect to Equality of opportunity (DEO).}
\label{fig:com}
\end{center}
\end{figure}

\begin{figure}[!t]
\begin{center}
\centerline{\includegraphics[width=0.95\columnwidth]{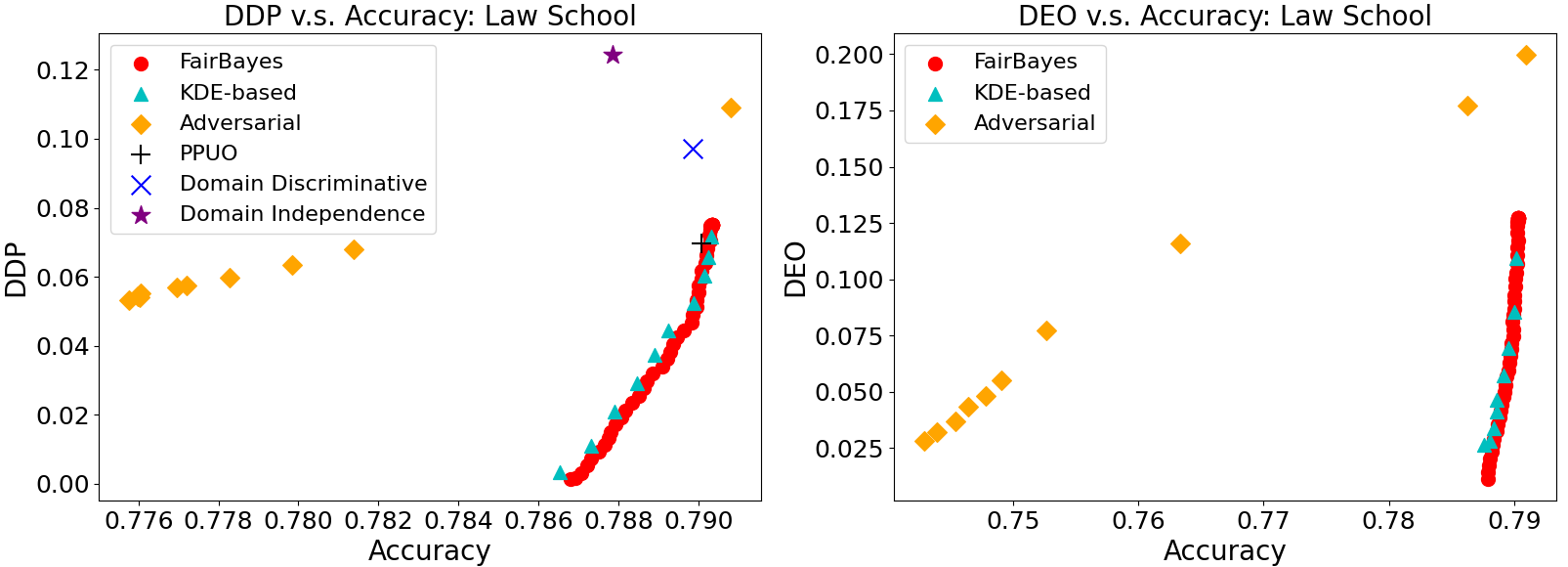}}
\caption{Fairness-Accuracy tradeoff on the ``Law school'' dataset. Left panel: Tradeoff with respect to demographic parity (DDP). Right panel: Tradeoff with respect to Equality of opportunity (DEO).}
\label{fig:Law}
\end{center}
\end{figure}

 On the COMPAS dataset, the domain-based methods achieve the best accuracy, but they do not succeed in controlling disparity. 
 As shown in Figure \ref{fig:com}, FairBayes achieves nearly the best performance on the fairness-accuracy tradeoff, while being slightly inferior to KDE based learning. This may caused by the relatively small size of the COMPAS dataset. 
 This small data size may result in a loss of accuracy when estimating $\eta_a(x)$, 
 which is used by  FairBayes to adjusts the per-class thresholds.   
 We emphasize that the success of  FairBayes relies on the consistent estimation of the per-group feature-conditional probabilities of the labels.
 
However, FairBayes is superior to other methods in computational speed, as the fairness-accuarcy tradeoff under different disparity levels or even different fairness measures can be generated based on only a single training run. The running times for various methods are
reported next.

\subsection{Running time analysis}
In this section, we compare the running time of various methods. All experiments are conducted on a personal computer with an Intel(R) Core(TM) i9-9920X CPU @ 3.50Ghz and an NVIDIA GeForce RTX 2080 Ti GPU. Table \ref{time:one} summarizes the average running time for different methods with one specific tuning knob. As we can see, domain-based training, FairBayes and post-processing unconstrained learning are more computationally efficient than the others. However, both domain-based training and post-processing unconstrained learning are only designed for perfect fairness while our FairBayes can handle any pre-determined disparity level.

\begin{table}[!h]
\caption{Average running time (seconds) for different methods aiming for demographic parity}
\label{time:one}
\begin{center}
\begin{small}
\begin{sc}
\begin{tabular}{c|c|c|c}
\hline
  Methods &Adult &  COMPAS & Law School  \\
\hline
FairBayes & 60 &16 &152\\
KDE & 147 &26 &201\\
Adversarial &206 & 51& 464\\
PPUO & 70 &18 &161\\
Domain\_Dis & 50 &11 &105\\
Domain\_Ind & 54 &12 &103\\\hline
\end{tabular}
\end{sc}
\end{small}
\end{center}
\end{table}

We further compare the running time of generating  fainess-accuracy tradeoff curves. For FairBayes, we consider 50 different disparity levels and for KDE-based learning and Adversarial based learning, we set 10 different tuning knobs. As we can see from Table  \ref{time:trade}, FairBayes is  significantly faster than the other two methods. 
Compared to the results in table \ref{time:one},  FairBayes spends almost the same time to derive a tradeoff curve, while the other two take about ten times more time.
In addition, FairBayes can further use the output from unconstrained learning  to generate fair classifiers under other fairness measures, such as equality of opportunity, which is  desirable when several different fairness measures are considered at the same time.

\begin{table}[!h]
\caption{Average running time (seconds) for deriving a tradeoff curve}
\label{time:trade}
\begin{center}
\begin{small}
\begin{sc}
\begin{tabular}{c|c|c|c}
\hline
  Methods & Adult&COMPAS&  Law School  \\
\hline
FairBayes & 66 &16 &159\\
KDE & 1415 &243 &1945\\
Adversarial &2279  &499 &4499\\\hline
\end{tabular}
\end{sc}
\end{small}
\end{center}
\end{table}

\subsection{Experiments with a Multi-class Protected Attribute}
In this section, we conduct more experiments to validate Theorem \ref{thm-opt-dp-multi}  and the FairBayes algorithm \ref{alg:multi} with a multi-class protected attribute. In this case, as in  \citep{CHS2020}, we use the following measure to quantify the level of unfairness with respect to demographic parity: 
$$\text{DDP}(f)=\sum_{a=1}^{|\mathcal{A}|}|P(\widehat{Y}_f  = 1|A = 1) -P(\widehat{Y}_f  = 1)|.$$
\subsubsection{Synthetic data}
We first consider a theoretical model where the theoretical fair Bayes-optimal classifier can be
derived explicitly to compare FairBayes with the theoretical benchmark.

We generate data $a\in\mathcal{A}=\{1,2,...,|\mathcal{A}|\}$ and $y\in\{0,1\}$ with, for $a\in\mathcal{A}$, $p_a=\frac{a^{1/2}}{\sum_{a=1}^{|\A|}a^{1/2}}$ and $p_{Y,a}\sim U(0,1)$, where $U(0,1)$ is the uniform distribution over $[0,1]$.
Conditional on $A=a$ and $Y=y$, $X$ is generated from a multivariate
Gaussian distribution $N(\mu_{ay},2^2I_{|\mathcal{A}|})$. Here, we set $\mu_{ay}=(2y-1)e_a$, where $e_a\in\mathbb{R}^{|\mathcal{A}|}$ is the unit vector with the $a$-th element equal to unity.

We consider three cases, $|\mathcal{A}|=3$, $|\mathcal{A}|=5$ and $|\A|=10$. For all cases, we set the training data sample size as $10000\times|\A|$ and the test data sample size as $5000$. As the Bayes-optimal classifier is linear in $x$, we use logistic regression to learn the group-wise conditional probabilities.

We repeat the experiments 100 times and present the simulation results in table \ref{table_synth_mu}. To better validate the performance of FairBayes, we also present the result for unconstrained optimization. As we can see, FairBayes effectively removes the disparity effect from the unconstrained classifier, and it closely tracks the behavior of the true fair Bayes-optimal classifier.

\begin{table}[!h]
\caption{Classification accuracy and DDP of the true fair Bayes-optimal classifier and our estimator trained via logistic regression.}
\label{table_synth_mu}
\vspace{-0.2cm}
\begin{center}
\setlength{\tabcolsep}{5.7pt}
\renewcommand{\arraystretch}{0.95}
\begin{small}
\begin{sc}
\begin{tabular}{cc|cc|cc|cc}
\hline
\multicolumn{4}{c|}{Fair classification}  & \multicolumn{4}{c}{Unconstrained classification} \\\hline
\multicolumn{2}{c|}{Theoretical}  &\multicolumn{2}{c|}{FairBayes}  &\multicolumn{2}{c|}{Theoretical}  &\multicolumn{2}{c}{Logistic regression}\\\hline

  $|\A|$ & ACC &DDP& ACC & DDP& ACC&{DDP}& {{ACC}} \\\hline
3& 0.615 & 0.021\,(0.011) & 0.616\,(0.008)  & 1.144 & 0.874& 1.142\,(0.014) & 0.875\,(0.004)\\
5 & 0.665 & 0.042\,(0.16) & 0.664\,(0.007) & 1.591 & 0.831 & 1.589\,(0.027) & 0.831\,(0.005)\\
10 & 0.615  & 0.190\,(0.042) & 0.615\,(0.008) & 3.688 & 0.824& 3.688\,(0.052) & 0.824\,(0.005)\\\hline
\end{tabular}
\end{sc}
\end{small}
\end{center}
\vspace{-0.4cm}
\end{table}

\subsubsection{Empirical Data}
Next, we conduct experiments on the  ``Adult'' dataset. This time, we have $|\A|=4$ with the protected attribute being the combination between ``Race'': white v.s. non-white and Gender: male v.s. female. We adopt the same experimental settings as for the binary protected attribute case and repeat the experiment 20 times. Table \ref{table_adult_multi} presents the simulation results for five different methods \footnote{We do not include post-processing unconstrained optimization as it is designed for a binary protected attribute.}. We observe that FairBayes outperforms other methods, achieving the best disparity control, almost the highest model accuracy and the smallest standard deviation of the performance metrics.

\begin{table}[!h]
 \caption{Classification accuracy and level of unfairness on the Adult dataset with a multi-class protected attribute.}
 \label{table_adult_multi}
 \vskip -0.1in
 \begin{center}
 \begin{small}
 \begin{sc}
 \begin{tabular}{c|c|ccc}\hline
 method &   parameter &ACC$_D$ & DDP &Runing time (seconds) \\\hline
 FairBayes & &0.831\,(0.001) & 0.030\,(0.013) & 57\\
KDE based &$\lambda=0.75$&0.826\,(0.006) &	0.039\,(0.024) &197\\
Adversarial&$\alpha=3$& 0.793\,(0.026) &	 0.221\,(0.190)  &213\\
 Domain dis & & 0.832\,(0.009)	& 0.143\,(0.053) &45	\\
 Domain ind& &0.831\,(0.011)   &0.175\,(0.070)	&50	\\
\hline
\end{tabular}
 \end{sc}
\end{small}
 \end{center}
 \vskip -0.2in
 \end{table}
 
\end{document}